\PassOptionsToPackage{pagebackref=true}{hyperref}
\documentclass{article}
\let\jmlrOriginalBibliographyStyle\bibliographystyle
\renewcommand{\bibliographystyle}[1]{%
 \jmlrOriginalBibliographyStyle{jmlr-visible}}
\usepackage[abbrvbib, preprint]{jmlr2e}
\let\bibliographystyle\jmlrOriginalBibliographyStyle
\PassOptionsToPackage{noabbrev}{cleveref}

\renewcommand*{\backref}[1]{}  %
\renewcommand*{\backrefalt}[4]{%
  \ifcase #1 %
    (Not cited.)%
  \or %
    (Cited on page #2.)%
  \else %
    (Cited on pages #2.)%
  \fi}

\usepackage[english]{babel}

\usepackage[letterpaper,top=2cm,bottom=2cm,left=3cm,
right=3cm,marginparwidth=1.75cm]{geometry}

\usepackage{graphicx}

\usepackage[table]{xcolor}
\definecolor{darkblue}{rgb}{0,0,.5}
\definecolor{darkorange}{rgb}{0.8, 0.4, 0}
\definecolor{caltechOrange}{rgb}{1.0, 0.424, 0.047}
\definecolor{caltechGreen}{RGB}{115,169,80}
\colorlet{PMS7489cLight}{caltechGreen!10!white}
\colorlet{PMS7489cFrame}{caltechGreen!75!black}
\colorlet{PMS7489cTitle}{caltechGreen!85!black}

\colorlet{highlight_links}{caltechOrange}

\usepackage{csquotes}

\usepackage{amsmath}
\usepackage{mathtools}
\usepackage{amsfonts}
\usepackage{mathrsfs}
\usepackage{amssymb}
\usepackage{dsfont} %
\usepackage{booktabs}

\usepackage{aligned-overset}
\usepackage{nicefrac}

\usepackage[T1]{fontenc}
\usepackage{enumitem}

\usepackage{pifont}
\usepackage{caption}
\usepackage{etoolbox}

\usepackage{cleveref}

\newtheorem{assumption}[theorem]{Assumption}
\newtheorem{dataAssumption}[theorem]{Data Assumption}
\newtheorem{approxAssumption}[theorem]{UA Assumption}

\numberwithin{theorem}{section}
\makeatletter
\let\c@example\c@theorem

\let\p@example\p@theorem
\makeatother

\newcommand{\statementendsymbol}{\ensuremath{\blacklozenge}}
\newcommand{\statementendmark}{%
 \ifmmode
  \quad\hbox{\statementendsymbol}%
 \else
  \leavevmode\unskip\nobreak\hfill\quad\hbox{\statementendsymbol}%
 \fi
}
\makeatletter
\newcommand{\uprightstatementbegin}{%
 \def\@begintheorem##1##2{\trivlist
  \item[\hskip\labelsep{\bfseries ##1\ ##2}]\normalfont}%
 \def\@opargbegintheorem##1##2##3{\trivlist
  \item[\hskip\labelsep{\bfseries ##1\ ##2\ (##3)}]\normalfont}%
}
\makeatother
\newcommand{\uprightstatementend}{\statementendmark}
\newcommand{\declareuprightstatement}[1]{%
 \AtBeginEnvironment{#1}{\uprightstatementbegin}%
 \AtEndEnvironment{#1}{\uprightstatementend}%
}
\declareuprightstatement{assumption}
\declareuprightstatement{dataAssumption}
\declareuprightstatement{approxAssumption}
\declareuprightstatement{definition}
\declareuprightstatement{example}
\declareuprightstatement{remark}

\crefname{theorem}{Theorem}{Theorems}
\crefname{lemma}{Lemma}{Lemmas}
\crefname{corollary}{Corollary}{Corollaries}
\crefname{assumption}{Assumption}{Assumptions}
\crefname{dataAssumption}{Data Assumption}{Data Assumptions}
\crefname{approxAssumption}{UA Assumption}{UA Assumptions}
\crefname{optiProblem}{Optimization Problem}{Optimization Problems}
\crefname{definition}{Definition}{Definitions}
\crefname{remark}{Remark}{Remarks}
\crefname{example}{Example}{Examples}

\newcommand{\assref}[1]{Assumption~\ref{#1}}
\newcommand{\Assref}[1]{Assumption~\ref{#1}}

\newcommand{\DataAssref}[1]{Data Assumption~\ref{#1}}

\newcommand{\UAAssref}[1]{UA Assumption~\ref{#1}}

\newcommand{\Remref}[1]{Remark~\ref{#1}}

\usepackage[colorinlistoftodos,prependcaption,textsize=tiny]{todonotes}

\usepackage[most]{tcolorbox} %
\newcounter{questionBoxCounter} %
\newcounter{takeAwayBoxCounter} %

\newtcolorbox{questionBox}[2][]{
 enhanced,
 colback=blue!5!white,         %
 colframe=blue!75!black,       %
 fonttitle=\bfseries,          %
 colbacktitle=blue!85!black,   %
 coltitle=white,               %
 arc=4pt,                      %
 boxrule=0.8pt,                %
 title={Question~\#\thequestionBoxCounter:~#2}, %
 attach boxed title to top left={xshift=1em, yshift=-2mm},
 boxed title style={
   enhanced,
   arc=3pt,                    %
   boxrule=0pt,
   top=0pt,
   bottom=0pt,
   left=0pt,
   right=0pt,
   interior style={fill=blue!85!black},
  },
 before upper={\stepcounter{questionBoxCounter}}, %
 #1
}

\newtcolorbox{takeAwayBox}[2][]{
 enhanced,
 colback=PMS7489cLight,       %
 colframe=PMS7489cFrame,      %
 fonttitle=\bfseries,         %
 colbacktitle=PMS7489cTitle,  %
 coltitle=white,              %
 arc=4pt,                     %
 boxrule=0.8pt,               %
 title={Key Takeaway~\#\thetakeAwayBoxCounter:~#2}, %
 attach boxed title to top left={xshift=1em, yshift=-2mm},
 boxed title style={
   enhanced,
   arc=3pt,
   boxrule=0pt,
   top=0pt,
   bottom=0pt,
   left=0pt,
   right=0pt,
   interior style={fill=PMS7489cTitle},
  },
 before upper={\stepcounter{takeAwayBoxCounter}}, %
 #1
}

\newtcolorbox{paperSummary}[1][]{
 enhanced,
 colback=PMS7489cLight,       %
 colframe=PMS7489cFrame,      %
 fonttitle=\bfseries,         %
 colbacktitle=PMS7489cTitle,  %
 coltitle=white,              %
 arc=4pt,                     %
 boxrule=0.8pt,               %
 left=-3mm,                   %
 title={The Short Version}, %
 attach boxed title to top left={xshift=1em, yshift=-2mm},
 boxed title style={
   enhanced,
   arc=3pt,
   boxrule=0pt,
   top=0pt,
   bottom=0pt,
   left=0pt,
   right=0pt,
   interior style={fill=PMS7489cTitle},
  },
}
\usetikzlibrary{
 positioning,
 calc,
 arrows.meta,
 decorations.pathreplacing,
 fit,
 backgrounds,
 patterns
}

\definecolor{oiblue}{RGB}{0,114,178}
\definecolor{oiorange}{RGB}{230,159,0}
\definecolor{oigreen}{RGB}{0,158,115}
\definecolor{oired}{RGB}{213,94,0}
\definecolor{oipurple}{RGB}{120,94,163}
\definecolor{lightblue}{RGB}{220,235,250}
\definecolor{lightorange}{RGB}{255,240,215}
\definecolor{lightgreen}{RGB}{215,245,230}
\definecolor{lightred}{RGB}{255,225,215}
\definecolor{lightpurple}{RGB}{232,225,245}
\definecolor{lightgray}{RGB}{245,245,245}
\definecolor{medgray}{RGB}{180,180,180}
\definecolor{darktext}{RGB}{40,40,40}

\newcommand{\innerEmpty}{\,\cdot\,}

\newcommand{\high}[1]{{\emph{#1}}}

\def\gNorm{{\bar{g}}}
\def\gHat{{\hat{g}}}
\def\thetaVec{{\boldsymbol{\theta}}}
\def\thetaVecHat{{\hat{\boldsymbol{\theta}}}}

\def\omegaHat{{\hat{\omega}}}

\newcommand{\dLatGen}[1]{d_{#1}^{\mathsf{Lat}}}
\def\dLat{{\dLatGen{}}}

\def\clipImage{{\mathrm{im}}}
\def\clipText{{\mathrm{txt}}}

\def\XX{{\mathcal{X}}}
\def\Borel{{\mathcal{B}}}

\def\SS{{\mathcal{S}}}
\def\Simplex{{\mathbb{S}}}
\def\RR{{\mathbb{R}}}
\def\PP{{\mathbb{P}}}

\def\FF{{\mathcal{F}}}
\def\NN{{\mathbb{N}}}

\def\CC{{\mathcal{C}}}

\def\xx{{\boldsymbol{x}}}

\def\KK{{\mathcal{K}}}
\def\MM{{\mathcal{M}}}
\def\app{{\mathsf{a}}}
\def\tilt{{\Phi}}
\def\cont{{\mathsf{c}}}
\def\nn{{\mathsf{NN}}}

\newcommand{\tiltContArg}[1]{{\tilt_{#1}^{\cont}}}

\def\prodMeas{{\mathsf{P}}}

\def\thetaHat{{\hat{\theta}}}

\DeclareMathOperator{\Uniform}{\mathcal{U}}
\DeclareMathOperator{\indicator}{\mathds{1}}
\DeclareMathOperator{\Prob}{\mathcal{P}}

\newcommand{\setN}[1]{[#1]}
\newcommand{\nfr}{\nicefrac}

\newcommand{\define}{\coloneqq}
\newcommand{\defineRev}{\eqqcolon}

\def\fusion{{\mathsf{fus}}}
\def\hadamard{{\mathsf{Had}}}
\def\twoclip{{\mathsf{2\text{-}clip}}}
\def\mclip{{\mathsf{m\text{-}clip}}}

\newcommand{\Exp}[1]{{ \mathbb{E}}\!\left[#1\right]}

\newcommand{\Expu}[2]{{\mathbb{E}}_{#1}\!\left[#2\right]}

\def\DD{{\mathsf{D}}}

\def\KLtag{{\mathsf{KL}}}

\def\TotVarTag{{\mathsf{TV}}}

\newcommand{\KL}[2]{
 {\DD}_{\KLtag}\!\left(#1 \,\middle\|\, #2\right)
}

\newcommand{\DTV}[2]{
 {\DD}_{\TotVarTag}\!\left(#1 \,\middle\|\, #2\right)
}

\def\lossSym{{\mathsf{L}}}

\DeclareMathOperator{\rank}{\mathrm{rk}}

\DeclareMathOperator{\Iden}{\mathrm{Id}}
\DeclareMathOperator{\Span}{\mathrm{span}}

\DeclareMathOperator*{\argmax}{arg\,max}

\newcommand*\Dd{\mathrm{d}} %
\newcommand*\dd{\: \mathrm{d}}
\newcommand*\dx{\: \mathrm{d}x}

\DeclarePairedDelimiter{\abs}{\lvert}{\rvert}
\DeclarePairedDelimiter{\norm}{\lVert}{\rVert}

\DeclarePairedDelimiterX{\IP}[2]{\langle}{\rangle}{#1, #2} %

\renewcommand{\phi}{\varphi}

\definecolor{byzantine}{rgb}{0.74, 0.2, 0.64}
\definecolor{darkgreen}{rgb}{0.1,0.6,0.1}
\definecolor{darkred}{rgb}{0.6,0,0}
\definecolor{darkorange}{rgb}{1.0, 0.55, 0.0}

\definecolor{ethblue}{rgb}{0.129, 0.361, 0.686}
\hypersetup{colorlinks=true, linkcolor=ethblue, citecolor=ethblue, urlcolor=ethblue}
\providecommand{\doi}[1]{}
\renewcommand{\doi}[1]{%
 \href{https://doi.org/#1}{\nolinkurl{doi:#1}}}
\providecommand{\arxiv}[1]{}
\renewcommand{\arxiv}[1]{%
 \href{https://arxiv.org/abs/#1}{\nolinkurl{arXiv:#1}}}

\usepackage[normalem]{ulem}

\begin{document}
\title{Expressivity In
 Multimodal Contrastive Learning}

\author{%
 \name Andrew Stuart \email
 astuart@caltech.edu \\ \addr California Institute of Technology \AND \name
 Florian Wolf\thanks{Corresponding author.} \email fwolf@caltech.edu \\
 \addr California Institute of Technology }

\editor{My editor}

\maketitle

\begin{abstract}
 Contrastive learning has become a cornerstone of modern representation
 learning, powering CLIP-style models that underpin text-to-image generation,
 vision-language models, and retrieval across a rapidly growing range of
 modalities. Despite this empirical success, the expressive power of these
 architectures remains poorly understood. To gain insight, we study
 expressivity by adopting a population-level, density-estimation viewpoint:
 each architecture comprises a parameterized set of densities whose
 parameters may be chosen to approximate the joint distribution of the
 modalities. This isolates a question of pure representational capacity: which joint
 distributions can a given contrastive family of parameterizations approximate
 to arbitrary accuracy? We show that expressivity is sharply architecture-dependent.
 For two modalities,
 the simple two-tower CLIP architecture is a universal approximator.
 A natural generalization of CLIP, widely used in practice when
 three or more modalities are present, is based on a loss found by summing
 over all pairwise similarities. This provably cannot represent arbitrary
 joint distributions, although we prove that it remains expressive enough to match all
 pairwise conditionals. Motivated by this gap, we propose Hadamard-CLIP,
 which adds a single learned weight vector on top of the existing encoders
 and restores universal approximation of the joint for any number of
 modalities while preserving CLIP's fast, precomputable-embedding retrieval.
\end{abstract}


\section{Introduction}
Multimodal contrastive learning refers to a class of algorithms designed to
find, in training, and exploit, at inference time, correlations between
different data modalities. The canonical example is text-image data, but
applications now demand the understanding of much broader classes of data
modalities, such as depth and segmentation maps or motion profiles, audio
and video, for example. This paper develops theory concerning the
expressivity, or limits on expressivity, of commonly used algorithms
deployed in multimodal contrastive learning. Furthermore, we propose a
new architecture which addresses shortcomings of
existing algorithms used in the case of three or more data modalities. In
Subsection \ref{ssec:BALR} we provide the background for our work, together
with a literature review. Subsection \ref{ssec:CAO} describes our
contributions and explains the organization of the remainder of the paper.
In Subsection \ref{ssec:NOT} we highlight notation used throughout.

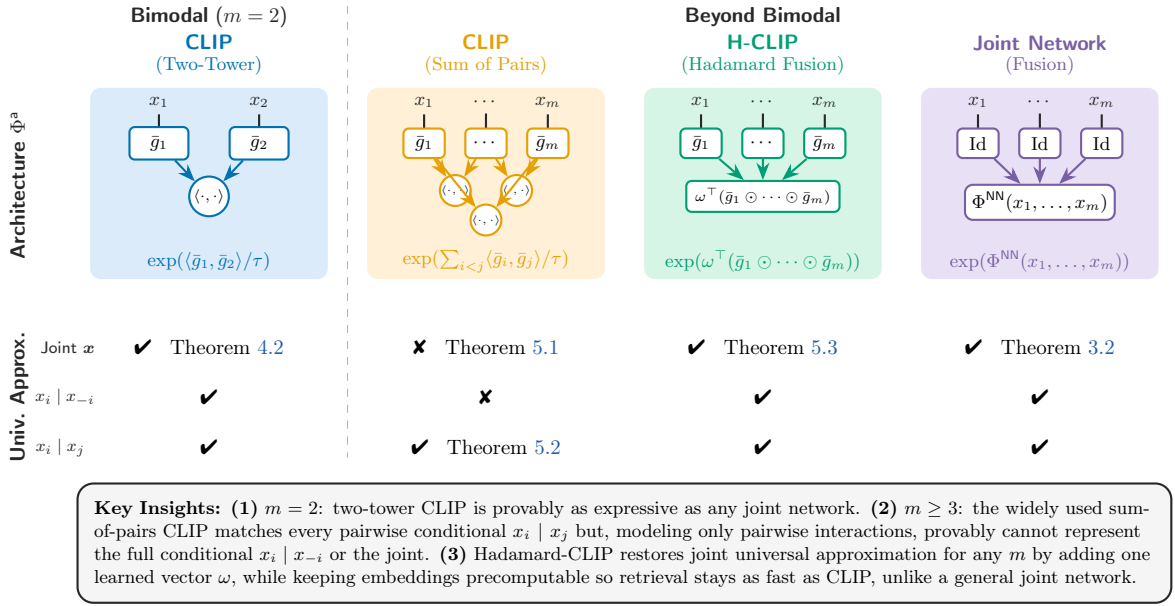
\begin{figure}
 \begin{center}
  \begin{tikzpicture}[
    every node/.style={font=\sffamily},
    >=Stealth,
    scale=0.78, transform shape,
   ]

   \def\colW{4.2}       %
   \def\colGap{0.5}     %
   \def\rowArch{0}      %
   \def\rowUA{-3.8}     %
   \def\headerY{2.0}    %
   \def\superY{3.2}     %
   \def\rowlabelX{-3.2} %
   \def\cellH{3.6}      %
   \def\cellHua{2.0}    %

   \pgfmathsetmacro{\colA}{0}
   \pgfmathsetmacro{\colB}{\colW+\colGap}
   \pgfmathsetmacro{\colC}{2*(\colW+\colGap)}
   \pgfmathsetmacro{\colD}{3*(\colW+\colGap)}

   \node[font=\sffamily, text=darktext] at (\colA, \superY-0.55) {\textbf{Bimodal} ($m=2$)};

   \pgfmathsetmacro{\midBC}{0.5*(\colB+\colD)}
   \node[font=\sffamily, text=darktext]
   at (\midBC, \superY-0.55) {\textbf{Beyond Bimodal}};

   \node[font=\sffamily\bfseries, text=oiblue, align=center]
   at (\colA, \headerY) {CLIP\\[-2pt]{\small\normalfont (Two-Tower)}};
   \node[font=\sffamily\bfseries, text=oiorange, align=center]
   at (\colB, \headerY) {CLIP\\[-2pt]{\small\normalfont (Sum of Pairs)}};
   \node[font=\sffamily\bfseries, text=oigreen, align=center]
   at (\colC, \headerY) {H-CLIP\\[-2pt]{\small\normalfont (Hadamard Fusion)}};
   \node[font=\sffamily\bfseries, text=oipurple, align=center]
   at (\colD, \headerY) {Joint Network\\[-2pt]{\small\normalfont (Fusion)}};

   \node[font=\sffamily\bfseries, text=darktext, rotate=90, anchor=south]
   at (\rowlabelX + 0.2, \rowArch - 0.2) {Architecture $\tilt^\app$};

   \fill[lightblue, rounded corners=4pt]
   ({\colA-0.48*\colW}, \rowArch+0.4*\cellH) rectangle ({\colA+0.48*\colW}, \rowArch-0.5*\cellH);
   \fill[lightorange, rounded corners=4pt]
   ({\colB-0.48*\colW}, \rowArch+0.4*\cellH) rectangle ({\colB+0.48*\colW}, \rowArch-0.5*\cellH);
   \fill[lightgreen, rounded corners=4pt]
   ({\colC-0.48*\colW}, \rowArch+0.4*\cellH) rectangle ({\colC+0.48*\colW}, \rowArch-0.5*\cellH);
   \fill[lightpurple, rounded corners=4pt]
   ({\colD-0.48*\colW}, \rowArch+0.4*\cellH) rectangle ({\colD+0.48*\colW}, \rowArch-0.5*\cellH);

   \begin{scope}[shift={(\colA, \rowArch-0.3)}]
    \node[draw=oiblue, thick, rounded corners=2pt, fill=white,
     minimum width=1.0cm, minimum height=0.55cm, font=\small]
    (g1) at (-0.85, 0.8) {$\gNorm_1$};
    \node[draw=oiblue, thick, rounded corners=2pt, fill=white,
     minimum width=1.0cm, minimum height=0.55cm, font=\small]
    (g2) at (0.85, 0.8) {$\gNorm_2$};
    \node[font=\small, text=darktext] (x1) at (-0.85, 1.5) {$x_1$};
    \node[font=\small, text=darktext] (x2) at (0.85, 1.5) {$x_2$};
    \draw[-, thick, darktext] (x1) -- (g1);
    \draw[-, thick, darktext] (x2) -- (g2);
    \node[draw=oiblue, thick, circle, inner sep=1pt, fill=white, font=\scriptsize]
    (ip) at (0, -0.1) {$\langle\cdot,\cdot\rangle$};
    \draw[->, thick, oiblue] (g1) -- (ip);
    \draw[->, thick, oiblue] (g2) -- (ip);
    \node[font=\small, text=oiblue] at (0, -1.2) {$\exp(\IP{\gNorm_1}{\gNorm_2}/\tau)$};
   \end{scope}

   \begin{scope}[shift={(\colB, \rowArch-0.3)}]
    \node[draw=oiorange, thick, rounded corners=2pt, fill=white,
     minimum width=0.7cm, minimum height=0.5cm, font=\small]
    (pg1) at (-1.05, 0.8) {$\gNorm_1$};
    \node[draw=oiorange, thick, rounded corners=2pt, fill=white,
     minimum width=0.7cm, minimum height=0.5cm, font=\small]
    (pg2) at (0, 0.8) {$\cdots$};
    \node[draw=oiorange, thick, rounded corners=2pt, fill=white,
     minimum width=0.7cm, minimum height=0.5cm, font=\small]
    (pg3) at (1.05, 0.8) {$\gNorm_m$};
    \node[font=\small, text=darktext] (px1) at (-1.05, 1.5) {$x_1$};
    \node[font=\small, text=darktext] (px2) at (0, 1.5) {$\cdots$};
    \node[font=\small, text=darktext] (px3) at (1.05, 1.5) {$x_m$};
    \draw[-, thick, darktext] (px1) -- (pg1);
    \draw[-, thick, darktext] (px2) -- (pg2);
    \draw[-, thick, darktext] (px3) -- (pg3);
    \node[draw=oiorange, thick, circle, inner sep=0pt, fill=white, font=\tiny]
    (pip12) at (-0.52, -0.0) {$\langle\cdot,\cdot\rangle$};
    \node[draw=oiorange, thick, circle, inner sep=0pt, fill=white, font=\tiny]
    (pip13) at (0, -0.5) {$\langle\cdot,\cdot\rangle$};
    \node[draw=oiorange, thick, circle, inner sep=0pt, fill=white, font=\tiny]
    (pip23) at (0.52, -0.0) {$\langle\cdot,\cdot\rangle$};
    \draw[->, thick, oiorange] (pg1) -- (pip12);
    \draw[->, thick, oiorange] (pg2) -- (pip12);
    \draw[->, thick, oiorange] (pg1) -- (pip13);
    \draw[->, thick, oiorange] (pg3) -- (pip13);
    \draw[->, thick, oiorange] (pg2) -- (pip23);
    \draw[->, thick, oiorange] (pg3) -- (pip23);
    \node[font=\small, text=oiorange] at (0, -1.2)
    {$\exp(\sum_{i<j}\IP{\gNorm_i}{\gNorm_j}/\tau)$};
   \end{scope}

   \begin{scope}[shift={(\colC, \rowArch-0.3)}]
    \node[draw=oigreen, thick, rounded corners=2pt, fill=white,
     minimum width=0.7cm, minimum height=0.5cm, font=\small]
    (hg1) at (-1.05, 0.8) {$\gNorm_1$};
    \node[draw=oigreen, thick, rounded corners=2pt, fill=white,
     minimum width=0.7cm, minimum height=0.5cm, font=\small]
    (hg2) at (0, 0.8) {$\cdots$};
    \node[draw=oigreen, thick, rounded corners=2pt, fill=white,
     minimum width=0.7cm, minimum height=0.5cm, font=\small]
    (hg3) at (1.05, 0.8) {$\gNorm_m$};
    \node[font=\small, text=darktext] (hx1) at (-1.05, 1.5) {$x_1$};
    \node[font=\small, text=darktext] (hx2) at (0, 1.5) {$\cdots$};
    \node[font=\small, text=darktext] (hx3) at (1.05, 1.5) {$x_m$};
    \draw[-, thick, darktext] (hx1) -- (hg1);
    \draw[-, thick, darktext] (hx2) -- (hg2);
    \draw[-, thick, darktext] (hx3) -- (hg3);
    \node[draw=oigreen, thick, rounded corners=2pt, fill=white,
     minimum width=2.6cm, minimum height=0.5cm, font=\scriptsize]
    (had) at (0, -0.1) {$\omega^\top(\gNorm_1 \odot \cdots \odot \gNorm_m)$};
    \draw[->, thick, oigreen] (hg1) -- (had);
    \draw[->, thick, oigreen] (hg2) -- (had);
    \draw[->, thick, oigreen] (hg3) -- (had);
    \node[font=\small, text=oigreen, align=center] at (0, -1.2)
    {$\exp(\omega^\top (\gNorm_1 \odot \cdots \odot \gNorm_m))$};
   \end{scope}

   \begin{scope}[shift={(\colD, \rowArch-0.3)}]
    \node[draw=oipurple, thick, rounded corners=2pt, fill=white,
     minimum width=0.7cm, minimum height=0.5cm, font=\small]
    (fg1) at (-1.05, 0.8) {$\Iden$};
    \node[draw=oipurple, thick, rounded corners=2pt, fill=white,
     minimum width=0.7cm, minimum height=0.5cm, font=\small]
    (fg2) at (0, 0.8) {$\Iden$};
    \node[draw=oipurple, thick, rounded corners=2pt, fill=white,
     minimum width=0.7cm, minimum height=0.5cm, font=\small]
    (fg3) at (1.05, 0.8) {$\Iden$};
    \node[font=\small, text=darktext] (fx1) at (-1.05, 1.5) {$x_1$};
    \node[font=\small, text=darktext] (fx2) at (0, 1.5) {$\cdots$};
    \node[font=\small, text=darktext] (fx3) at (1.05, 1.5) {$x_m$};
    \draw[-, thick, darktext] (fx1) -- (fg1);
    \draw[-, thick, darktext] (fx2) -- (fg2);
    \draw[-, thick, darktext] (fx3) -- (fg3);
    \node[draw=oipurple, thick, rounded corners=3pt, fill=white,
     minimum width=2.2cm, minimum height=0.5cm, font=\small, align=center]
    (fuse) at (0, -0.2) {$\tilt^\nn(x_1, \ldots, x_m)$};
    \draw[->, thick, oipurple] (fg1) -- (fuse);
    \draw[->, thick, oipurple] (fg2) -- (fuse);
    \draw[->, thick, oipurple] (fg3) -- (fuse);
    \node[font=\small, text=oipurple, align=center] at (0, -1.2)
    {$\exp(\tilt^\nn(x_1, \ldots, x_m))$};
   \end{scope}

   \pgfmathsetmacro{\subRowH}{0.85}
   \pgfmathsetmacro{\cellHsub}{0.75}
   \pgfmathsetmacro{\rowUAjoint}{\rowUA + \subRowH}
   \pgfmathsetmacro{\rowUAfull}{\rowUA}
   \pgfmathsetmacro{\rowUApair}{\rowUA - \subRowH}

   \node[font=\sffamily\bfseries, text=darktext, rotate=90, anchor=south]
   at (\rowlabelX+0.26, \rowUA) {Univ.\ Approx.};

   \node[font=\sffamily\footnotesize, text=darktext, anchor=east]
   at ({\colA-0.48*\colW+0.2}, \rowUAjoint) {Joint $\xx$};
   \node[font=\sffamily\footnotesize, text=darktext, anchor=east]
   at ({\colA-0.48*\colW+0.2}, \rowUAfull) {$x_i\mid x_{-i}$};
   \node[font=\sffamily\footnotesize, text=darktext, anchor=east]
   at ({\colA-0.48*\colW+0.2}, \rowUApair) {$x_i\mid x_j$\;\;};

   \node[font=\normalsize, text=black] at (\colA, \rowUAjoint)
   {\ding{52} \;\;\Cref{thm:NormalizedTwoClipUniversalApproximation}};
   \node[font=\normalsize, text=black] at (\colB, \rowUAjoint)
   {\ding{56} \;\;\Cref{thm:noUniversalApproximationClip}};
   \node[font=\normalsize, text=black] at (\colC, \rowUAjoint)
   {\ding{52} \;\;\Cref{thm:UniversalApproximationHadamardClip}};
   \node[font=\normalsize, text=black] at (\colD, \rowUAjoint)
   {\ding{52} \;\;\Cref{result:UniversalApproxJointNetwork}};

   \node[font=\normalsize, text=black] at (\colA, \rowUAfull) {\ding{52}};
   \node[font=\normalsize, text=black] at (\colB, \rowUAfull) {\ding{56}};
   \node[font=\normalsize, text=black] at (\colC, \rowUAfull) {\ding{52}};
   \node[font=\normalsize, text=black] at (\colD, \rowUAfull) {\ding{52}};

   \node[font=\normalsize, text=black] at (\colA, \rowUApair) {\ding{52}};
   \node[font=\normalsize, text=black] at (\colB, \rowUApair) {\ding{52}
    \;\;\Cref{thm:UniversalApproxClipPairwiseConditionals}};
   \node[font=\normalsize, text=black] at (\colC, \rowUApair) {\ding{52}};
   \node[font=\normalsize, text=black] at (\colD, \rowUApair) {\ding{52}};

   \pgfmathsetmacro{\vsepX}{0.5*(\colA + 0.48*\colW + \colB - 0.48*\colW)}
   \draw[medgray, thin, dashed]
   (\vsepX, \superY-0.1) -- (\vsepX, \rowUA - \subRowH - 0.5*\cellHsub + 0.2);

   \pgfmathsetmacro{\boxW}{\colD + 0.48*\colW - \colA + 0.48*\colW}
   \pgfmathsetmacro{\boxCx}{0.5*(\colA + \colD)}
   \node[draw=darktext, thick, rounded corners=6pt, fill=lightgray,
    text width=18.0cm, align=left, inner sep=8pt, font=\small]
   at (\boxCx, \rowUA - \subRowH - 0.5*\cellHsub - 1.25)
   {\textbf{Key Insights:}
    \textbf{(1)} $m=2$: two-tower CLIP is provably as expressive as any joint network.
    \textbf{(2)} $m\geq 3$: the widely used sum-of-pairs CLIP matches every pairwise conditional
    $x_i\mid x_j$ but, modeling only pairwise interactions, provably cannot represent the full
    conditional $x_i\mid x_{-i}$ or the joint.
    \textbf{(3)} Hadamard-CLIP restores joint universal approximation for any $m$ by adding
    one learned vector $\omega$, while keeping embeddings precomputable
    so retrieval stays as fast as CLIP, unlike a general joint network.
   };
  \end{tikzpicture}
 \end{center}
 \caption{
  \textbf{Graphical Abstract.} We frame contrastive learning as learning a tilting function
  $\tilt^\app$, a score for how much more likely a tuple $(x_1,\dots,x_m)$ is to co-occur than under independence
  (the log-density of the joint relative to the product of marginals).
  \textbf{Top row:} Four ways to build $\tilt^\app$ from per-modality
  encoders, with the induced similarity score below each
  ($\tau$: learnable temperature, $\omega$: learned fusion weights).
  \textbf{Bottom table:} Which distributions each architecture can approximate arbitrarily well.
  The full joint $\xx$, one modality given all others ($x_i\mid x_{-i}$), and one given
  just one other ($x_i\mid x_j$); \ding{52} = provably can, \ding{56} = provably cannot.}
 \label{fig:GraphicalAbstract}
\end{figure}

\subsection{Background and Literature Review} \label{ssec:BALR}

We situate our work through a literature review 
tracing the progression from the canonical bimodal CLIP
setting to broader multimodal contrastive learning. We first introduce CLIP and
its empirical impact, then discuss existing theory and the population-level
viewpoint used in this paper, before turning to extensions beyond two
modalities and to connections with unimodal contrastive learning.

\paragraph{Contrastive Learning and CLIP.}
Contrastive Language-Image Pretraining
\citep{radfordLearningTransferableVisual2021} and related approaches such
as ALIGN \citep{jiaScalingVisualVisionLanguage2021} demonstrate that
contrastive objectives trained on large-scale multimodal data can learn
highly transferable joint representations across modalities.

A natural motivating example for two data modalities is image and text
pairs. We view an image (resp. text) as an element in $\XX_{\clipImage}$
(resp. $\XX_{\clipText}$). Then consider the question of learning about
alignment between image and text from $N$ data samples
$\{(x_{\clipImage}^{(n)}, x_{\clipText}^{(n)})\}_{n=1}^N$ drawn from a
joint distribution on $\XX_{\clipImage} \times \XX_{\clipText}$. CLIP does
so by learning (normalized) encoders $\gNorm_{\clipImage},
 \gNorm_{\clipText}$, typically constrained to take values on the unit
sphere, that produce aligned representations of both modalities in a shared
latent space by assigning high cosine similarity scores to related pairs
and low cosine similarity scores to unrelated pairs. The training objective
is a symmetric cross-entropy loss over similarity scores of the form
\begin{align*}
 \lossSym \define
 - \frac{1}{N} \sum_{n=1}^{N} \Biggl(
 \log \frac{\exp(\IP{\gNorm_{\clipImage}(x_{\clipImage}^{(n)})}{\gNorm_{\clipText}(x_{\clipText}^{(n)})} / \tau)}{
  \frac{1}{N}\sum_{k={1}}^{N}\exp(\IP{\gNorm_{\clipImage}(x_{\clipImage}^{(k)})}{\gNorm_{\clipText}(x_{\clipText}^{(n)})} / \tau)}
 + \log \frac{\exp(\IP{\gNorm_{\clipImage}(x_{\clipImage}^{(n)})}{\gNorm_{\clipText}(x_{\clipText}^{(n)})} / \tau)}{
  \frac{1}{N}\sum_{k={1}}^N\exp(\IP{\gNorm_{\clipImage}(x_{\clipImage}^{(n)})}{\gNorm_{\clipText}(x_{\clipText}^{(k)})} / \tau)}\Biggr),
\end{align*}
where temperature $\tau > 0$ is learned along with the parameters
defining the encoders $\gNorm_{\clipImage},\gNorm_{\clipText}$.
We define the scaled cosine similarity scores $s_\theta(x_{\clipImage},x_{\clipText})
 \define \IP{\gNorm_{\clipImage}(x_{\clipImage})}{
  \gNorm_{\clipText}(x_{\clipText})}/\tau$.
Intuitively, the objective encourages aligned
embeddings for related data pairs while separating unrelated pairs that
appear as contrastive examples in the denominator of the loss.

Beyond zero-shot classification and retrieval, CLIP-style representations
and contrastive loss functions have shown tremendous empirical success and
have become a central building block in modern multimodal machine learning.
CLIP encoders are now used in text-conditioned \emph{diffusion models}
\citep{rameshZeroShotTexttoImageGeneration2021,
 rameshHierarchicalTextConditionalImage2022, DALLE32023,
 esserScalingRectifiedFlow2024}, as \emph{vision backbones}
\citep{liuVisualInstructionTuning2023,
 liBLIP2BootstrappingLanguageImage2023, mizrahi4MMassivelyMultimodal2023,
 liuImprovedBaselinesVisual2024a, linMixedAttentionNetwork2024,
 wangCogVLMVisualExpert2024a, bachmann4M21AnytoAnyVision2024}, and
\emph{semantic backbones} \citep{luddeckeImageSegmentationUsing2022a,
 shridharCLIPortWhatWhere2022}. The contrastive loss has found wide
application in the training of \emph{Vision-Language-Models},
\emph{foundation models} and \emph{tokenizers}
\citep{liAlignFuseVision2021, alayracFlamingoVisualLanguage2022,
 yuCoCaContrastiveCaptioners2022, mustafaMultimodalContrastiveLearning2022,
 kwonMaskedVisionLanguage2022,
 wuSemanticEquivalenceTokenization2024,
 wuNExTGPTAnytoAnyMultimodal2024,
 GPT4VisionSystemCard2024,
 wuVILAUUnifiedFoundation2024, luATokenUnifiedTokenizer2025,
 maUniTokUnifiedTokenizer2025, schlarmannFuseLIPMultimodalEmbeddings2025,
 zhaoQLIPTextAlignedVisual2025}. Moreover, contrastive representation
learning has increasingly been extended \emph{beyond vision and language},
for example in neural analysis
\citep{schneiderLearnableLatentEmbeddings2023} and in sequential recommender
systems \citep{zhouS^3RecSelfSupervisedLearning2020,
 xieContrastiveLearningSequential2021,
 weiMultilevelCrossmodalContrastive2024}

\paragraph{Existing Theory.}
The empirical success of contrastive learning has motivated a substantial
body of theoretical work aimed at understanding the mechanisms underlying
representation learning with contrastive objectives. Nevertheless, a
complete theoretical understanding remains elusive. One important line of
research deals with an \emph{information theoretic} perspective on
contrastive learning. In its classical formulation shown above, CLIP is
using a two-sided InfoNCE loss function
\citep{oordNeuralDiscreteRepresentation2018} which has been shown to
maximize a lower bound on the mutual information (MI) between positive
pairs \citep{hjelmLearningDeepRepresentations2018a,
 pooleVariationalBoundsMutual2019,
 bachmanLearningRepresentationsMaximizing2019,
 tianContrastiveMultiviewCoding2020a, wuMutualInformationContrastive2020,
 wangUnderstandingContrastiveRepresentation2022,
 shwartzzivCompressNotCompress2024, guiMultimodalContrastiveLearning2025}.
However, \cite{tschannenMutualInformationMaximization2020} show that
maximizing MI alone is not sufficient for good downstream performance.

Driven by the success of CLIP encoders as backbones for a wide variety of
tasks, a further line of research analyzes \emph{downstream-task
 performance guarantees} for CLIP representations.
\cite{tianContrastiveMultiviewCoding2020a} first showed empirically that
multimodal learning is beneficial, and \cite{huangWhatMakesMultiModal2021}
subsequently proved, under a composite assumption, that multimodal learning
achieves a smaller population risk than a unimodal representation. Since
then, a variety of works have studied zero-shot transfer in CLIP
\citep{chenUnderstandingTransferableRepresentation2023,
 mayilvahananDoesCLIPsGeneralization2023, uesakaWeightedPointSet2024,
 okoStatisticalTheoryContrastive2025,
 liStatisticalConsistencyGeneralization2026}. Motivated by the limited
expressivity of the cosine similarity, \cite{uesakaWeightedPointSet2024}
propose a weighted kernel function in the latent space that requires a
smaller latent dimensionality to achieve comparable performance. In a
follow-up work, \cite{yoshidaTheoreticalRefinementCLIP2025} prove, from a
Reproducing Kernel Hilbert Space perspective and under
conditional-independence and Lipschitz-continuity assumptions, that the
pointwise mutual information can be approximated to arbitrary accuracy.

Generalizing results from the unimodal literature (see below),
\cite{daunhawerIdentifiabilityResultsMultimodal2022} establish
\emph{identifiability results} for the recovery of hidden latent variables.
From the same angle, and assuming a linear data-generating process,
\cite{renImportanceContrastiveLoss2023} analyze the \emph{training
 dynamics} and show that the negative samples balance the condition number
and help avoid dimensional collapse. Concurrently,
\cite{nakadaUnderstandingMultimodalContrastive2023} prove that the gradient
dynamics used to minimize the CLIP loss are equivalent to applying a
singular value decomposition to the contrastive cross-covariance matrix.
Characterizing the population limit of the dynamics of contrastive
learning, \citep{furusawaMeanFieldTheory2023a,
 mengTrainingDynamicsNonlinear2024} analyze the mean-field limit of
gradient-based training.

\paragraph{Population-level Perspective.}
A perspective central to this work is the \emph{population-level}
formulation of \citet{baptistaMathematicalPerspectiveContrastive2025},
based on the learning problem defined in the $N=\infty$ limit. This casts
CLIP as density estimation: the given data is now the joint law $\mu$ of
the modalities, rather than a finite collection of samples, and from this
one seeks to approximate the joint through a learnable tilting of the
product of the marginals $\mu_\prodMeas$. The relevant object is the
log-density ratio $\log\frac{\Dd\mu}{\Dd\mu_\prodMeas}$, a \emph{tilting
 function} that measures how much more likely a tuple is to co-occur than
under independence; it vanishes exactly when the modalities are
independent, and for $m=2$ it is the pointwise mutual information. CLIP
parameterizes it as $\exp(\IP{\gNorm_\clipImage}{\gNorm_\clipText}/\tau)$.
At finite sample size, CLIP normalizes similarity scores using empirical
averages over mismatched image--text pairs. In the population limit, these
empirical averages become integrals with respect to the marginal laws, so
the same score \(s_\theta\) defines conditional model distributions. We
make this connection explicit in \Cref{ex:CLIP}. We adopt this viewpoint
throughout (formalized in \Cref{sec:Setting}): each architecture
parameterizes a score that, exponentiated, induces an approximate model
measure $\mu^{\app}(\innerEmpty{};\theta)$, and \emph{expressivity} asks
which tiltings, equivalently which joint distributions, the score class can
represent. Our measure of expressivity is its best achievable value
$\inf_{\theta\in\Theta} \KL{\mu}{\mu^{\app}(\innerEmpty{};\theta)}$. As an
infimum over the whole parameter space at the population measure, it
lower-bounds the error of any estimator on any sample, isolating
representational capacity from finite-sample estimation, optimization, and
generalization. Finally we note that the methods of analysis used to
quantify the difference between $\mu$ and $\mu^{\app}$, in terms of the
approximation of their denstities with respect to a common reference
measure, are closely related to methods developed to study approximations
of the Bayesian posterior distribution in KL divergence or in the Hellinger
metric \citep{marzouk2009stochastic,stuart2010inverse}.

\paragraph{Contrastive Learning Beyond the Bimodal Setting.}
Aligning and fusing information across heterogeneous modalities is a
central, still-open challenge in representation learning
\citep{yuanSurveyMultimodalLearning2025, liMultimodalAlignmentFusion2026},
and contrastive objectives have emerged as a core approach. Motivated by
the success of CLIP in the image--text setting, numerous extensions have
been proposed for learning from more than two modalities. Existing
approaches predominantly follow two strategies. One is to aggregate
pairwise contrastive interactions across modalities through a
\emph{sum-of-pairs} objective \citep{tianContrastiveMultiviewCoding2020a,
 portillo-quinteroStraightforwardFrameworkVideo2021a,
 luoCLIP4ClipEmpiricalStudy2022, guzhovAudioclipExtendingClip2022,
 nagraniLearningAudioVideoModalities2022, tevetMotionCLIPExposingHuman2022,
 zhuBringingMultimodalityAmazon2024, zhouTENTConnectLanguage2026}.
Alternatively, for larger modality collections or when data from certain
modalities is scarce, many architectures instead introduce one or several
\emph{binding modalities}, often images or text, and contrast all remaining
modalities against these anchors \citep{girdharImageBindOneEmbedding2023,
 luATokenUnifiedTokenizer2025}. Despite their empirical success, the
theoretical understanding of these multimodal formulations remains limited.

\paragraph{Relation to Unimodal Contrastive Learning.}
Many unimodal self-supervised learning methods, such as those in
\citep{heMomentumContrastUnsupervised2020a,
 chenSimpleFrameworkContrastive2020, chenBigSelfsupervisedModels2020,
 grillBootstrapYourOwn2020}, can be interpreted within the same framework by
viewing data augmentations as generating multiple correlated views of a
latent sample. In this perspective, unimodal contrastive learning becomes a
special case of multimodal contrastive learning with augmentation-induced
conditional distributions.

Theoretical analyses in the unimodal setting have primarily focused on
three questions. The first concerns \emph{latent recovery and
 identifiability}, where contrastive objectives are studied through explicit
generative models and connections to nonlinear independent component
analysis \citep{hyvarinenIndependentComponentAnalysis2004,
 comonHandbookBlindSource2010, hyvarinenNonlinearICAUsing2019a,
 zimmermannContrastiveLearningInverts2021,
 vonkugelgenSelfSupervisedLearningData2021, reizingerCrossEntropyAllYou2024,
 hyvarinenIdentifiabilityLatentvariableStructuralequation2024,
 rusakInfoNCEIdentifyingGap2025}. The second investigates \emph{performance
 guarantees for downstream tasks}, including analyses of negative sampling
\citep{nozawaUnderstandingNegativeSamples2021a,
 awasthiMoreNegativeSamples2022, ashInvestigatingRoleNegatives2022,
 baoSurrogateGapContrastive2022, leiGeneralizationAnalysisContrastive2023,
 zhangAugmentationOverlapTheory2025} and data augmentation techniques
\citep{saunshiTheoreticalAnalysisContrastive2019, leePredictingWhatYou2021,
 wenUnderstandingFeatureLearning2021, wangChaosLadderNew2021,
 toshContrastiveEstimationReveals2021, toshContrastiveLearningMultiview2021,
 huangGeneralizationContrastiveSelfSupervised2022,
 wangRethinkingMinimalSufficient2022,
 wangUnderstandingContrastiveRepresentation2022, jiPowerContrastFeature2023,
 dufumierIntegratingPriorKnowledge2023,
 linStatisticalTheoryContrastive2025a, elstTightPACBayesianRisk2025}.
The third direction studies the \emph{optimization and geometric
 properties} of self-supervised learning objectives, e.g. by establishing
connection to distributionally robust optimization
\citep{wuUnderstandingContrastiveLearning2023}, (spectral) clustering
\citep{haochenProvableGuaranteesSelfSupervised2021,
 tanContrastiveLearningSpectral2023} and using gradient flow techniques to
analyze dimensionality collapse
\citep{jingUnderstandingDimensionalCollapse2021,
 ziyinRotationalSymmetryLoss2022,
 huangTheoreticalAnalysisSelfSupervised2025} or develop new loss functions
\citep{tianUnderstandingDeepContrastive2022a}.

\paragraph{Relationship of our Work to Existing Literature.}
Across all these strands, the analyses characterize what contrastive
learning \emph{recovers} (latent identifiability), what it \emph{bounds}
(mutual information), or what its optimization \emph{converges to}
(training dynamics); yet each presupposes that the chosen parameterization
\emph{can} represent the target relationship in the first place, and the
expressiveness of the contrastive parameterization itself has received
comparatively little attention. We therefore study a complementary, and
foundational, question: independently of finite-sample estimation,
optimization and generalization, \emph{is the contrastive parameterization
 itself expressive enough to universally approximate the underlying joint
 distribution, or only certain conditionals?}

\subsection{Contribution and Paper Organization} \label{ssec:CAO}
The primary contributions of this paper are universal approximation
theorems, and examples of failure of universal approximation, for various
multimodal contrastive learning algorithms. We work entirely in the
population limit, thereby excluding finite data-size effects, allowing us
to concentrate on expressivity. The resulting universal approximation
results are summarized in \Cref{fig:GraphicalAbstract}. In detail, we make
the following contributions, structured to simultaneously describe the
paper organization:
\begin{enumerate}[start=1,label={(C\arabic*)}]
 \item \emph{KL upper bounds for conditional losses}
       (\Cref{sec:Setting}).
       Contrastive learning typically deploys loss functions defined through matching conditional
       distributions. We show that the joint KL divergence between the data-generating measure
       and its parametric approximation strictly upper-bounds both the full
       and pairwise conditional losses used in contrastive
       learning algorithms. This motivates formulation of subsequent expressivity
       results in terms of the KL divergence.
 \item \emph{Function-to-measure universal approximation}
       (\Cref{sec:GeneralUA}).
       We establish that universal approximation at the level of measures, and with respect to KL divergence
       between the true joint and the approximate joint, is implied by uniform approximation of the tilting function,
       lifting function-level guarantees to measure-level guarantees.
 \item \emph{Bimodal CLIP is as expressive as any joint network}
       (\Cref{sec:CLIP}).
       For $m=2$, we prove that the two-tower CLIP architecture universally
       approximates any joint distribution, matching the expressiveness of a
       general joint network in the approximation sense, for both
       unnormalized and (with learnable temperature) normalized encoders.
 \item \emph{For $m \geq 3$, CLIP fails at the joint but succeeds at pairwise
        conditionals} (\Cref{sec:ContrastiveLearningBeyondBimodal}).
       We show that the widely used sum-of-pairs CLIP generalization
       is expressive enough to
       universally approximate all pairwise conditionals simultaneously,
       but provably cannot approximate the joint.
 \item \emph{Hadamard-CLIP restores universal approximation}
       (\Cref{sec:ContrastiveLearningBeyondBimodal}).
       We propose H-CLIP (Hadamard-CLIP), which adds a single learnable weight
       vector $\omega \in \RR^{\dLat}$ applied to the element-wise product of
       encoder outputs, restoring joint universal approximation for any $m$
       while preserving CLIP's inference-time efficiency.
\end{enumerate}

In order to increase readability, we have placed much of the 
technical lemmas, and details of the proofs of theorems, 
in a six section appendix. These are referred to from within the 
relevant sections of the paper; and the introduction to the appendix
explains how each subsection relates to the main body of the paper.

\subsection{Notation} \label{ssec:NOT}
Throughout this work, we define $\NN_{\geq 2} \define \{n\in \NN \;\vert \;
 n\geq 2\}$ and denote by $m \in \NN_{\geq 2}$ the number of data
modalities. We use $\xx = \bigl(x_1^\top, \ldots, x_m^\top\bigr)^\top \in
 \RR^d$ to denote concatenation of (possibly random) vectors; here $x_i \in
 \RR^{d_i}$, $d_i \in \NN$ for $i=1, \ldots, m$, and we denote by $d =
 \sum_{i=1}^{m} d_i$ the total dimension. For dimension $\ell \in \NN$, we
denote by $\Simplex^{\ell -1} \define \{y \in \RR^\ell \; \vert\;
 \norm{y}_2 = 1\}$ the unit sphere. For integer $k\in \NN$ we write
$\setN{k} \define \{1, \ldots, k\}.$ We denote the set of probability
measures defined on a sample space $\SS$ by $\Prob(\SS)$. We use the
following notation for summations interchangeably $\sum_{i=1}^{m}
 \sum_{j=i+1}^{m} = \sum_{1 \leq i < j \leq m } = \sum_{i < j}$, in contexts
where the bounds on the variables $i, j$ are clear. For $\XX \subset
 \RR^d$, a $\sigma$-algebra $\FF$ over $\XX$, a measure $\mu : \FF \to
 \RR_{\geq 0} \cup \{\infty\}$, and a $\mu$-measurable function $f : \XX \to
 \RR$, define $\norm{f}_{L^p} \define \left(\int_{\XX} \abs{f(\xx)}^p \;
 \mu(\Dd \xx)\right)^{\nicefrac{1}{p}}$, $p\in [1,\infty)$ and extend to the
case $p=\infty$ through essential supremum. We denote with $L^p(\XX, \mu)
 \define \{f:\XX \to \RR \;\vert\; \norm{f}_{L^p} < \infty\}$ the space of
measurable functions with finite $L^p$-norm. We denote by $\lambda$ the
Lebesgue measure on Euclidean space $\RR^k$ of any dimension $k$. For a
measure $\mu$ on $(\XX, \Borel(\XX))$ and a $\mu$-measurable function $f:
 \XX \to \RR^d$, we denote by $f_{\#} \mu$ the push-forward measure of $\mu$
under $f$. For $i \in \setN{m}$, let $\pi_{-i}$ denote the projection
 of $\xx \in \XX$ onto all components except $x_i$. For distinct $i,j \in
  \setN{m}$, let $\pi_{ij}$ denote the projection onto $(x_i,x_j)$. For any
 $\nu \in \Prob(\XX)$, define the corresponding marginal measures $\nu_{-i}
  \define (\pi_{-i})_{\#}\nu$ and $\nu_{ij} \define (\pi_{ij})_{\#}\nu$; in
 particular, $\mu_{-i}$ and $\mu_{ij}$ are the corresponding marginals of
 $\mu$.

\section{Probabilistic Formulation}
\label{sec:Setting}

In this section we provide the probabilistic formulation that underpins the
approximation theoretical questions that are the heart of this paper. In
\Cref{subsec:SettingProblemFormulation} we specify the data model and
define our goal as approximating the joint distribution with respect to
some statistical divergence. We motivate the form of population-level loss
functions arising in practical algorithms in \Cref{subsec:BCME}, based on
matching conditionals. Then, using this form,
\Cref{subsec:SettingMultimodalContrastiveLearning} provides justification
of the choice of the KL divergence, showing that it upper-bounds the
conditional losses used in practical contrastive learning methods.

\subsection{Problem Formulation}
\label{subsec:SettingProblemFormulation}
We start by stating the problem setting, the main assumption on our data
distribution and the overarching goals of this paper.

\begin{assumption}[Problem Setting \& Data Distribution]
 \label[assumption]{ass:ProblemSettingDataDistribution}
 Denote by $\XX_i$ a subset of $\RR^{d_i}, d_i \in \NN$ and by
 $\XX \define \prod_{i=1}^m \XX_i \subset \RR^{d}.$ Throughout this work,
 we assume the following properties of $\mu$, the data generating probability
 measure connecting $m\in \NN_{\geq 2}$ modalities.

 \begin{enumerate}[start=1,label={(A\arabic*)}]
  \item \label{ass1:MarginalDistribution}\emph{Marginal Distributions.}
        For each modality $i \in \setN{m}$, $\mu$ has a marginal distribution
        $\mu_i \in \Prob(\XX_i)$ with support
        $\XX_i.$ The measure $\mu_i$
        is absolutely continuous with respect to the
        Lebesgue measure restricted to $\XX_i:$  $\lambda_i \define \lambda\vert_{\XX_i}$; and
        $\mu_i$ has Radon-Nikodym
        derivative $\nicefrac{\Dd \mu_i}{\Dd \lambda_i} = f_i \in L^1(\XX_i,
         \lambda_i)$.
  \item \label{ass2:JointDistribution}\emph{Joint Distribution.}
        The space $\XX$ supports the product measure $\mu_\prodMeas = \bigotimes_{i=1}^m \mu_i$. The data generating measure
        $\mu \in \Prob(\XX)$ of $\xx := (x_1, \ldots, x_m)$ is defined by a density
        $r:\XX \to [0,\infty]$ via the Radon-Nikodym-derivative
        $\nicefrac{\Dd \mu}{\Dd \mu_\prodMeas} = r \in L^1(\XX, \mu_\prodMeas).$ We refer to
        $r$ as a \emph{tilting} of the data generating measure.
  \item \label{ass3:FiniteMutualInformation}\emph{Finite Mutual Information.}
        We assume that $\log r \in L^1(\XX, \mu)$.
 \end{enumerate}
\end{assumption}
With the conventions that $\log(0) = -\infty$ and $\exp(-\infty) = 0$, we
can write the change of measures as
\begin{align}
 \frac{\Dd \mu}{\Dd \mu_\prodMeas}(\xx) = \exp(\tilt(\xx)),
 \tag{\ensuremath{\tilt}-Tilting}
\end{align}
where we refer to $\tilt \define \log(r): \XX \to \RR \cup \{\pm\infty\}$
as \emph{the tilting function of $\mu$}.
We will see in the following sections that working with log density
$\tilt$ is more convenient in terms of coherence with the literature
as well as for our theoretical analysis. By
\ref{ass3:FiniteMutualInformation}, $\tilt \in L^1(\XX, \mu)$.

In this paper we study the approximation of measure $\mu$ by a measure
$\mu^\app$ induced by an \emph{approximate tilting function} $\tilt^\app:
 \XX \to \RR$:
\begin{align}
 \begin{aligned}
  \frac{\Dd\mu^\app(\Dd \xx)}{\Dd \mu_\prodMeas} (\xx) & \define \frac{1}{Z^\app} \cdot
  \exp(\tilt^\app(\xx))                                                                 \\
  Z^\app                                               & \define \int_{\XX}
  \exp(\tilt^\app(\xx)) \;\mu_\prodMeas(\Dd \xx).
 \end{aligned}
 \tag{\ensuremath{\tilt^\app}-Tilting}
 \label{eq:GeneralTiltingFunction}
\end{align}
Approximations of this form are at the heart of contrastive learning
methods to study alignment between different data modalities; in particular
expressing the true data generating measure, and its approximation, via
densities with respect to the product of the marginals of the data
distribution, leads to actionable algorithms deployed in practice.

\begin{remark}[Notation] \label[remark]{rem:tilt}
 We use $\mu^\app$ as shorthand for $\mu^{\tilt^{\app}}$ defined by an
 approximating tilting function $\tilt^{\app}$. This notational convention
 is extended as follows: when an architecture fixes a particular tilting
 function we propagate the superscript from the tilting function into the
 approximate measure. For example, when using \(\tilt^{\twoclip}\), as in
 the bimodal CLIP methodology illustrated in \Cref{ex:CLIP}, we write \(\mu^\twoclip\)
 for \(\mu^{\tilt^\twoclip}\). The same notational convention is applied to the
 normalization constants arising from the approximate tiltings.
\end{remark}

The goal of this paper is to quantify how well $\mu^\app$ approximates
$\mu$ by estimating $\DD(\mu,\mu^\app)$, where
$\DD:\Prob(\XX)\times\Prob(\XX)\to\RR$ is a metric, or statistical
divergence, to be specified. We approach this problem by studying the
approximation of the tilting function $\tilt$ by $\tilt^\app$, chosen from
a given function class, and studying the effect of this approximation on
the approximation of $\mu$ by $\mu^\app.$ These function classes are
selected to provide a deeper understanding of multimodal contrastive
learning algorithms used in practice and, in the case $m\geq 3$, to suggest
variants of these algorithms with improved approximation capabilities.

Importantly, while the true tilting function $\tilt$ may be
extended-real-valued (taking values in $\RR$ as well as $\pm \infty$),
every approximate tilting function
$\tilt^\app:\XX\to\RR$ is finite-valued. Consequently,
$\Dd\mu^\app/\Dd\mu_\prodMeas>0$ on $\XX$, so
$\mu\ll\mu^\app$ and we may apply the Radon-Nikodym (RN) chain rule
\begin{align}
 \frac{\Dd \mu}{\Dd \mu^\app} = \frac{\Dd \mu}{\Dd \mu_\prodMeas}
 \cdot \left(\frac{\Dd \mu^\app}{\Dd \mu_\prodMeas}\right)^{-1},
 \label{eq:RadonNikodymChainRule}
 \tag{RNCR}
\end{align}
an identity which will be used throughout this work.

Our theoretical studies will be conducted in the population limit, with
theory expressed in terms of measure $\mu$. In practical machine learning
settings, however, one only has access to finitely many samples drawn from
$\mu$ and loss functions are implemented through empirical approximations
based on finite data. To bridge this gap between population-level theory
developed in this work and the sample-based objectives employed in
practice, we introduce the following data assumption; thereby connecting
our theoretical framework and practitioner-oriented objectives discussed in
the introduction. This connection is made in the following subsection.

\begin{dataAssumption}
 \label[dataAssumption]{ass:MainDataAssumption}
 We assume to have access to $N\in \NN$ data
 samples $\{\xx^{(n)}\}_{n=1}^N \overset{\mathrm{iid}}{\sim} \mu$
 and decompose each data point into different modalities as
 $\xx^{(n)} = \bigl((x_1^{(n)})^\top, \ldots, (x_m^{(n)})^\top\bigr)^\top.$
 We define the following empirical measures $\mu^{N} \define \frac{1}{N}
  \sum_{n=1}^{N} \delta_{\xx^{(n)}}$, $\mu_i^N \define
  \frac{1}{N}\sum_{n=1}^{N} \delta_{x^{(n)}_i}$, $i=1,\ldots, m$, and
 $\mu^N_\prodMeas \define \bigotimes_{i=1}^m \mu_i^N$,
 approximating the population level joint,
 marginal and product distributions, respectively.
\end{dataAssumption}

\subsection{Population-Level Multimodal Constrastive Loss Function}
\label{subsec:BCME}

Consider the following loss function, based on the full conditional
distributions of $\mu$ and $\mu^\app$:
\begin{equation}
 \frac{1}{m} \sum_{i=1}^{m}
 \Expu{x_{-i}\sim \mu_{-i}}{\KL{\mu_{x_i\vert x_{-i}}(\innerEmpty{}\vert x_{-i})}{
   \mu^\app_{x_i\vert x_{-i}}(\innerEmpty{}\vert x_{-i})
  }}.
 \tag{Full-Cond}
 \label{eq:plcl}
\end{equation}
Here $\mu^\app$ is a class of approximate joint measures found from
approximate tilting function $\Phi^\app.$ The following example shows that,
in the large data limit, the
\emph{Contrastive Language-Image Pre-training (CLIP)} methodology,
introduced in \cite{radfordLearningTransferableVisual2021}, is an instance of this
class of population-level loss functions, arising when $m=2$.

\begin{example}[CLIP, $m=2$]\label{ex:CLIP}
 In CLIP both modalities, for example image and text, are encoded \emph{separately}
 into a shared latent space $\Simplex^{\dLat-1}$, using modality-specific
 normalized encoders $\gNorm_i(\innerEmpty{};\theta_i) : \XX_i \to
  \Simplex^{\dLat-1}$, $i=1, 2$, parametrized by $\thetaVec = (\theta_1,
  \theta_2) \in \Theta \subseteq \RR^p$; cf. \Cref{fig:GraphicalAbstract}. 
  The derivation below follows \cite[Thm.
  3.2]{baptistaMathematicalPerspectiveContrastive2025};
  we include it because it makes explicit how the population-level conditional 
  KL objective recovers the standard empirical CLIP loss, a concept which is
  key in our developments here.
 We begin with the \emph{separable} tilting function
 \begin{align}
  \tilt^{\twoclip}(\xx; \thetaVec) \define
  \IP{\gNorm_1(x_1; \theta_1)}{\gNorm_2(x_2;\theta_2)}/\tau.
  \tag{\ensuremath{\tilt^\twoclip}-Tilting}
  \label{eq:ClipTiltingFunctionBiModal}
 \end{align}
 Invoking the notational conventions introduced in \Remref{rem:tilt},
 this defines $\mu^{\twoclip}(\innerEmpty{};\thetaVec).$
 Expanding the conditional-based contrastive loss defined by the left-hand
 side of the identity \eqref{eq:JointKLBoundsFullConditionals} with $m=2$
 yields
 \begin{align*}
  \lossSym^{\twoclip}(\thetaVec)
   & = \frac{1}{2}\Expu{x_1 \sim \mu_1}{\KL{\mu_{x_2\vert x_1}}{\mu^{\twoclip}_{x_2\vert x_1}}}
  + \frac{1}{2}\Expu{x_2 \sim \mu_2}{\KL{\mu_{x_1\vert x_2}}{\mu^{\twoclip}_{x_1\vert x_2}}}    \\
   & = -\frac{1}{2}\Expu{x_1 \sim \mu_1}{
   \Expu{x_2 \sim \mu_{x_2\vert x_1}}{
    \log \frac{\exp(\tilt^{\twoclip}_{x_2\vert x_1}(x_2\vert
  x_1;\thetaVec))}{ Z^{\twoclip}_{x_2\vert x_1}(\thetaVec)}}}                                   \\
   & \qquad\qquad\qquad\qquad-\frac{1}{2} \Expu{x_2 \sim\mu_2}{
   \Expu{x_1 \sim \mu_{x_1\vert x_2}}{\log
    \frac{\exp(\tilt^{\twoclip}_{x_1\vert x_2}(x_1\vert x_2;\thetaVec))}{ Z^{\twoclip}_{
  x_1\vert x_2}(\thetaVec)}}} + \mathsf{const},                                                 \\
   & = -\frac{1}{2}\Expu{(x_1,x_2) \sim \mu}{
   \log \frac{\exp(\tilt^{\twoclip}_{x_2\vert x_1}(x_2\vert
  x_1;\thetaVec))}{ Z^{\twoclip}_{x_2\vert x_1}(\thetaVec)}}                                    \\
   & \qquad\qquad\qquad\qquad-\frac{1}{2} \Expu{(x_1,x_2) \sim\mu}{\log
   \frac{\exp(\tilt^{\twoclip}_{x_1\vert x_2}(x_1\vert x_2;\thetaVec))}{ Z^{\twoclip}_{
      x_1\vert x_2}(\thetaVec)}} + \mathsf{const},
 \end{align*}
 where $\mathsf{const}$ is finite due to
 \assref{ass3:FiniteMutualInformation}, and is independent of the parameter $\thetaVec$.
 Next, we use \DataAssref{ass:MainDataAssumption} and replace the measures
 with its empirical counterparts resulting in the sum of two cross entropy
 losses of (empirical) conditional distributions
 \begin{align}
  \begin{aligned}
   \lossSym^{\twoclip}_N(\thetaVec)
    & \define - \frac{1}{N}\sum_{n=1}^{N}\Biggr(
   \log \frac{\exp(\IP{\gNorm_1(x_1^{(n)};\theta_1)}{\gNorm_2(x_2^{(n)};\theta_2)})}{
   \frac{1}{N}\sum_{k=1}^{N}
   \exp(\IP{\gNorm_1(x_1^{(k)};\theta_1)}{\gNorm_2(x_2^{(n)};\theta_2)})} \\
    & \qquad\qquad\qquad\qquad +
   \log \frac{\exp(\IP{\gNorm_1(x_1^{(n)};\theta_1)}{\gNorm_2(x_2^{(n)};\theta_2)})}{
   \frac{1}{N}\sum_{l=1}^{N}
   \exp(\IP{\gNorm_1(x_1^{(n)};\theta_1)}{\gNorm_2(x_2^{(l)};\theta_2)})}\Biggl),
  \end{aligned}
  \tag{\ensuremath{\lossSym^{\twoclip}_N}-Loss}
  \label{eq:BiModalClipLoss-Empirical}
 \end{align}
 exactly recovering the empirically used CLIP loss described in the introduction.
\end{example}

Going beyond the bimodal setting of $m=2$ in \Cref{ex:CLIP}, a central
limitation of the full conditional loss formulation in \eqref{eq:plcl} is
the availability of high-quality data across all modalities. In large-scale
multimodal learning settings, datasets are typically constructed by
aggregating heterogeneous sources, and it is therefore rare that every
modality is present for every data sample. Consequently, for multimodal
problems with $m \ge 3$, most state-of-the-art contrastive (pre-)training
losses heavily rely on artificially generated data or are formulated in
terms of partially observed modality combinations
\citep{zamirTaskonomyDisentanglingTask2018,
 eftekharOmnidataScalablePipeline2021,
 ghiasiMultiTaskSelfTrainingLearning2021,
 robertsHypersimPhotorealisticSynthetic2021,
 bachmannMultiMAEMultimodalMultitask2022, mizrahi4MMassivelyMultimodal2023,
 bachmann4M21AnytoAnyVision2024}. This takes us beyond the setting of Data
Assumption \ref{ass:MainDataAssumption}.

Accordingly, for each pair $(i,j)$, $i\neq j$, entering the loss, we assume access
 to i.i.d.\ paired samples from the marginal $\mu_{ij}$; samples need not be
 aligned across pairs or arise from fully observed multimodal data points.
 Under this \emph{weaker sampling regime}, a natural and particularly appealing
 extension of
\eqref{eq:BiModalClipLoss-Empirical} is to consider contrastive pairs with
respect to pairwise conditionals of the form $x_i \vert x_j$, for $i\neq
 j$, only requiring the presence of two modalities at the same time, within
the data set. Empirically, variations of this loss are effectively used for
contrastive (pre-)training beyond bimodal (image, text) pairings
\citep{portillo-quinteroStraightforwardFrameworkVideo2021a,
 luoCLIP4ClipEmpiricalStudy2022, guzhovAudioclipExtendingClip2022,
 nagraniLearningAudioVideoModalities2022, tevetMotionCLIPExposingHuman2022,
 girdharImageBindOneEmbedding2023, zhuBringingMultimodalityAmazon2024,
 zhouTENTConnectLanguage2026} and recently also as a part of the loss
function in the training of (multimodal) tokenizers
\citep{luATokenUnifiedTokenizer2025}.
Similarly to \Cref{ex:CLIP}, one can show that this approach corresponds to
a pairwise conditional loss function using a separable tilting function
\begin{align}
 \tilt^{\mclip}(\xx; \thetaVec) \define
 \sum_{1\leq i < j \leq m} \IP{\gNorm_i(x_i; \theta_i)}{\gNorm_j(x_j;\theta_j)}/\tau
 \tag{\ensuremath{\tilt^\mclip}-Tilting}
 \label{eq:ClipTiltingFunctionGeneral}
\end{align}
for \emph{normalized} encoders $\gNorm_i(\innerEmpty{};\theta_i):\XX_i \to
 \Simplex^{\dLat -1}$ of the modality $i\in \setN{m}$,
a (learnable) temperature $\tau > 0$ and a parameter vector
$\thetaVec \define (\theta_1^\top, \ldots, \theta_m^\top)^\top \in \Theta \subseteq \RR^p$.
Again, adopting the notation from Remark \ref{rem:tilt},
these considerations lead to the specific loss function
\begin{align*}
 \begin{aligned}
  \frac{1}{2m(m-1)} \sum_{1 \leq i < j \leq m}
  \biggl(
   & \Expu{x_j\sim \mu_j}{\KL{\mu_{x_i\vert x_j}}{\mu^\mclip_{x_i\vert x_j}}}
  + \Expu{x_i\sim \mu_i}{\KL{\mu_{x_j\vert x_i}}{\mu^\mclip_{x_j\vert x_i}}}
  \biggr),
 \end{aligned}
 \notag
\end{align*}
and to the generic version
\begin{align*}
 \begin{aligned}
  \frac{1}{2m(m-1)} \sum_{1 \leq i < j \leq m}
  \biggl(
   & \Expu{x_j\sim \mu_j}{\KL{\mu_{x_i\vert x_j}}{\mu^\app_{x_i\vert x_j}}}
  + \Expu{x_i\sim \mu_i}{\KL{\mu_{x_j\vert x_i}}{\mu^\app_{x_j\vert x_i}}}
  \biggr),
 \end{aligned}
 \tag{Sum-Pair-Cond}
 \label{eq:plcl2}
\end{align*}
as an alternative to \eqref{eq:plcl}.

\subsection{Multimodal Contrastive Learning}
\label{subsec:SettingMultimodalContrastiveLearning}

In this subsection we show the value of choosing $\DD(\innerEmpty{},
 \innerEmpty{})$ to be a KL divergence. In particular making this divergence
small, between $\mu$ and $\mu^\app$, yields control over the conditional
distributions of $\mu$ and $\mu^\app$, and control over the losses defined
in \eqref{eq:plcl} and \eqref{eq:plcl2}.

\begin{theorem}[Joint Controls Full Conditionals Loss]
 \label{thm:JointVsFullConditionals}
 Since $\mu\ll \mu^\app$ by construction, for every modality $i \in
  \setN{m}$, the following inequality holds:
 \begin{align*}
  \Expu{x_{-i}\sim \mu_{-i}}{\KL{\mu_{x_i\vert x_{-i}}(\innerEmpty{}\vert x_{-i})}{
    \mu^\app_{x_i\vert x_{-i}}(\innerEmpty{}\vert x_{-i})
   }}
  \leq \KL{\mu}{\mu^\app}.
 \end{align*}
 In particular, averaging all of these terms yields
 \begin{align}
  \frac{1}{m} \sum_{i=1}^{m}
  \Expu{x_{-i}\sim \mu_{-i}}{\KL{\mu_{x_i\vert x_{-i}}(\innerEmpty{}\vert x_{-i})}{
    \mu^\app_{x_i\vert x_{-i}}(\innerEmpty{}\vert x_{-i})
   }}
  \leq \KL{\mu}{\mu^\app}.
  \tag{Joint-Full}
  \label{eq:JointKLBoundsFullConditionals}
 \end{align}
\end{theorem}
\begin{proof}
 See \Cref{thm:JointVsFullConditionals-PROOF}.
\end{proof}
The theorem shows that minimizing the joint loss $\KL{\mu}{\mu^\app}$ is
strictly stronger than
minimizing the \eqref{eq:plcl} loss obtained by averaging the conditional
losses over the conditioning variable and summing them over all data modalities.
The following theorem shows that similar ideas apply to control of the
pairwise loss function \eqref{eq:plcl2}, and its constituents.
In fact, there is an underlying hierarchy of theorems of this type,
as made explicit in the remark following the theorem.

\begin{theorem}[Joint Controls Pairwise Conditionals Loss]
 \label{thm:JointVsPairwiseConditionals}
 Since $\mu \ll \mu^\app$ by construction, for every modality pair
 $(i,j) \in \setN{m}^2$, $i \neq j$, the following inequality holds
 \begin{align*}
  \Expu{x_j\sim \mu_j}{\KL{\mu_{x_i\vert x_j}}{\mu^\app_{x_i\vert x_j}}}
  \leq \KL{\mu}{\mu^\app}.
 \end{align*}
 In particular, averaging all of these terms yields
 \begin{equation*}
  \frac{1}{2m(m-1)} \sum_{1 \leq i < j \leq m}
  \left(
  \Expu{x_j\sim \mu_j}{\KL{\mu_{x_i\vert x_j}}{\mu^\app_{x_i\vert x_j}}}
  + \Expu{x_i\sim \mu_i}{\KL{\mu_{x_j\vert x_i}}{\mu^\app_{x_j\vert x_i}}}
  \right)                                         \\
  \leq \KL{\mu}{\mu^\app}.
 \end{equation*}
\end{theorem}
\begin{proof}
 See \Cref{thm:JointVsPairwiseConditionals-PROOF}.
\end{proof}

\begin{remark}[Arbitrary Conditionals]
 Combining the proof techniques of the two previous theorems,
 one can generalize the upper bound of the joint KL divergence
 to any subset of conditionals of cardinality $k \in \setN{m-1}$.
\end{remark}
Having shown the value of controlling the KL divergence
between $\mu$ and $\mu^\app$, we now proceed to obtain general
universal approximation results, expressed in terms of KL divergence,
for probability measures. We achieve this by employing standard
universal approximation results for functions,
applying them to log densities and deploying these results
to control the KL divergence.

\section{General Universal Approximation}
\label{sec:GeneralUA}
We perform our general analysis under a universal approximation
\Assref{ass:DensityParametricFamilyInContinuousFunctions}, below,
on a parametric function class $\MM^\nn$. This assumption plays
two distinct roles:
\begin{itemize}
 \item \emph{Abstract level.} In \Cref{result:UniversalApproxJointNetwork}
       and \Cref{result:UniversalApproxJointNetworkTightness} we show that
       if $\MM^\nn$ universally approximates continuous functions,
       then the family of measures induced by \eqref{eq:GeneralTiltingFunction}
       with a $\tilt^\nn \in \MM^\nn$ as the tilting function
       universally approximates the data-generating measure $\mu$
       with respect to the KL divergence.
 \item \emph{Architectural level.}
       The abstract result reduces universal approximation of $\mu$ to
       verifying \UAAssref{ass:DensityParametricFamilyInContinuousFunctions}.
       For certain architectures, we show that
       if each modality-specific encoder satisfies
       \UAAssref{ass:DensityParametricFamilyInContinuousFunctions},
       then the full model inherits this property as well.
       Concretely, \Cref{sec:CLIP} carries out this
       verification for bimodal CLIP. On the other hand,
       \Cref{sec:ContrastiveLearningBeyondBimodal} contains
       demonstration that the natural
       extension \eqref{eq:ClipTiltingFunctionGeneral} to $m \geq 3$
       modalities fails to satisfy the assumption. 
       This motivates
       using the abstract results to guide us to an architectural
       adaptation which addresses this problem, also in
       \Cref{sec:ContrastiveLearningBeyondBimodal}.
\end{itemize}
\begin{approxAssumption}[Universal Approximation in
  $(C^0(\XX; \RR), \norm{\innerEmpty{}}_{\infty})$]
 \label[approxAssumption]{ass:DensityParametricFamilyInContinuousFunctions}
 We assume to have access to a parametric family of continuous functions
 \begin{align*}
  \MM^\nn = \bigcup_{p \in \NN} \{\tilt^\nn(\innerEmpty{};\theta) \in C^0(\XX; \RR) \;\vert \;
  \theta \in \Theta \subseteq \RR^p\}
 \end{align*}
 satisfying $C^0$-\emph{universal approximation}: for every continuous
 function $\tilt \in C^0(\XX;\RR)$, every $\epsilon>0$, and every compact
 $\KK \Subset \XX$, there exists $p\in \NN$ and a $\thetaHat =
  \thetaHat(\epsilon, \KK) \in \Theta$ such that
 \begin{align}
  \norm{\tilt - \tilt^\nn(\innerEmpty{};\thetaHat)}_{\infty, \KK}
  \define \sup_{x \in \KK} \;\abs{\tilt(\xx) - \tilt^\nn(\xx;\thetaHat)} < \epsilon.
  \tag{\ensuremath{C^0}-UA}
  \label{eq:UniversalApproxDensityFunction}
 \end{align}
\end{approxAssumption}
Classical parametric families like polynomials and Fourier-product bases
satisfy \UAAssref{ass:DensityParametricFamilyInContinuousFunctions} 
\citep{powell1981approximation,lightApproximationTheoryTensor1985,
cheneyCourseApproximationTheory2009,
trefethenApproximationTheoryApproximation2019}
furthermore a wide variety of literature verifies this property
for neural network-based architecture. This latter body of work includes
a wide range of architectures all the way from single layer
perceptrons \citep{pinkusApproximationTheoryMLP1999} through to the recent
developments of self- and cross-attention based approximations
\citep{liuAttentionMechanismMaxAffine2025a}.
We refer to \citep{kriegApproximationFunctionsOptimal2026} for a
recent overview of sampling and complexity guarantees for function
approximation.

We state two versions of the universal approximation results. First, in the
setting of $\XX$ being a compact set and second, a generalization without
the assumption of compactness.
\begin{theorem}[Joint Network UA, Compact Support]
 \label{result:UniversalApproxJointNetwork}
 Consider measure $\mu$ satisfying \Assref{ass:ProblemSettingDataDistribution}.
 Assume $\XX \Subset \RR^{d}$ is compact, and
 let $\MM^\nn$ be a parametric family of functions satisfying
 \UAAssref{ass:DensityParametricFamilyInContinuousFunctions}.
 Then, for every $\epsilon > 0$, there exists $p \in \NN$ and
 a parameter $\thetaHat \in \Theta \subseteq \RR^p$ with a corresponding tilting function
 $\tilt^\nn(\innerEmpty{};\thetaHat)$ such that the measure
 $\mu^\nn(\innerEmpty{};\thetaHat)$ induced by
 \eqref{eq:GeneralTiltingFunction} satisfies
 \begin{align*}
  \KL{\mu}{\mu^\nn(\innerEmpty{};\thetaHat)} \leq \epsilon.
 \end{align*}
\end{theorem}
\begin{proof}
 We describe the overarching structure of the proof here, appealing to results in the
 appendix for details on the key components.
 Let $\epsilon > 0$ be fixed and define $\epsilon^\prime \define \epsilon/4$. We
 invoke \Cref{lem:LusinsTheoremGeneric} with $\epsilon^\prime$ yielding a constant
 $M = M(\epsilon^\prime)$.
 We pick $\eta \define \frac{1}{\exp(M(\epsilon^\prime) + \epsilon^\prime)}
  \cdot \epsilon^\prime$ and obtain a continuous representative $\tilt^\cont$
 of the tilting function $\tilt$. By \UAAssref{ass:DensityParametricFamilyInContinuousFunctions},
 we can find $\thetaHat \in \Theta$ and a corresponding $\tilt^\nn(\innerEmpty{}, \thetaHat)$
 such that
 $\norm{\tilt^\cont - \tilt^\nn(\innerEmpty{};\thetaHat)}_{\infty} \leq \frac{\epsilon}{4}$
 and \Cref{prop:FuncToMeasCompactSupport} yields
 \begin{align*}
  \KL{\mu}{\mu^\nn(\innerEmpty{};\thetaHat)} \leq \epsilon^\prime
  + \epsilon^\prime + \epsilon^\prime +
  \exp(M(\epsilon^\prime) + \epsilon^\prime) \cdot \eta
  \leq \epsilon,
 \end{align*}
 concluding the claim by the choice of $\epsilon^\prime$ and $\eta$.
\end{proof}
Next, we generalize this result to the setting of $\mu$ having (potentially)
non-compact support.
\begin{theorem}[Joint Network UA, Tightness]
 \label{result:UniversalApproxJointNetworkTightness}
 Consider measure $\mu$ satisfying \Assref{ass:ProblemSettingDataDistribution}.
 Assume $\XX$ is a (not necessarily compact) subset of a Euclidean space, and
 let $\MM^\nn$ be a parametric family of functions satisfying
 \UAAssref{ass:DensityParametricFamilyInContinuousFunctions}.
 Then for every $\epsilon > 0$, there exists $p\in \NN$ and a parameter
 $\thetaHat \in \Theta \subseteq \RR^p$ with a corresponding tilting function
 $\tilt^\nn(\innerEmpty{};\thetaHat)$ such that the measure
 $\mu^\nn(\innerEmpty{};\thetaHat)$ induced by
 \eqref{eq:GeneralTiltingFunction} satisfies
 \begin{align*}
  \KL{\mu}{\mu^\nn(\innerEmpty{};\thetaHat)} \leq \epsilon.
 \end{align*}
\end{theorem}
\begin{proof} We describe the overarching structure of the proof here, appealing to results in the appendix for details establishing the key components.
 Define $\epsilon^\prime \define \frac{\epsilon}{6}$.
 Define the finite Borel measure
  $\nu(A)\define\int_A\abs{\tilt}\dd\mu$. Since $\mu$, $\mu_\prodMeas$, and
  $\nu$ are tight \cite[Lem. 13.5]{klenkeProbabilityTheoryComprehensive2020},
  we can choose $\KK\Subset\XX$ such that their masses outside $\KK$ are all
  smaller than $\epsilon^\prime$. Applying
  \Cref{lem:LusinsTheoremGeneric} to the restrictions of $\mu$ and
  $\mu_\prodMeas$ to $\KK$, with parameter $\epsilon^\prime$, yields
  $\tilt^\cont\in C^0(\KK;\RR)$ satisfying the remaining hypotheses of
  \Cref{prop:FuncToMeasTightness}.
 On $\KK$,
 \UAAssref{ass:DensityParametricFamilyInContinuousFunctions} yields
 $\thetaHat(\epsilon^\prime) \in \Theta$ satisfying
 $\norm{\tilt^\cont-\tilt^\nn(\innerEmpty{};\thetaHat)}_{\infty,\KK}
  \leq \epsilon^\prime.$
 Then, abusing notation to avoid proliferation of symbols, we
 redefine $\tilt^\nn$ via the tail-truncation
 \begin{align*}
  \tilt^\nn(\xx;\thetaHat) \define \indicator_{\KK}(\xx) \cdot \tilt^\nn(\xx;\thetaHat).
 \end{align*}
 Invoking \Cref{prop:FuncToMeasTightness} with its parameters
 $\epsilon=\eta=\delta=\epsilon^\prime$, $B=0$, and
  $\tilt^\app=\tilt^\nn(\innerEmpty{};\thetaHat(\epsilon^\prime))$ yields
 \begin{align*}
  \KL{\mu}{\mu^\nn(\innerEmpty{};\thetaHat)}
  \leq 2\epsilon^\prime+2\epsilon^\prime+2\epsilon^\prime
  =\epsilon,
 \end{align*}
 which concludes the claim.
\end{proof}

\section{CLIP: Contrastive Learning $m=2$}
\label{sec:CLIP}
In this section, we will show that the CLIP architecture in the bimodal
case fulfills universal approximation of measures. We formulate the results
for compactly supported measures but, analogously to
\Cref{result:UniversalApproxJointNetworkTightness}, they can be generalized
to the non-compact setting.  The first theorem
shows universal approximation in the case of unnormalized encoders,
while the second one uses normalized encoders, explicitly requiring
the temperature $\tau > 0$ to be learnable.
\begin{theorem}[\eqref{eq:ClipTiltingFunctionBiModal} UA, General Encoders]\leavevmode
 \label{lem:DensityClipUnnormalizedEncoders}
 Consider measure $\mu$ satisfying \Assref{ass:ProblemSettingDataDistribution}.
 Assume that $\XX_1 \Subset \RR^{d_1}, \XX_2 \Subset \RR^{d_2}$ are compact, and
 let $\MM^\nn$ be a parametric family satisfying
 \UAAssref{ass:DensityParametricFamilyInContinuousFunctions}.
 Then, for every $\epsilon > 0$, there exist a latent dimension $\dLat \in
  \NN$, $p\in \NN$, and parameters $\thetaVecHat = (\thetaHat_1, \thetaHat_2)
  \in \Theta \subseteq \RR^p$
 with corresponding encoders $g_i(\innerEmpty{};\thetaHat_i) : \XX_i \to
  \RR^{\dLat}$, $i=1,2$, defining \eqref{eq:ClipTiltingFunctionBiModal} such
 that the measure induced by \eqref{eq:GeneralTiltingFunction} satisfies
 \begin{align*}
  \KL{\mu}{\mu^{\twoclip}(\innerEmpty{};\thetaVecHat)} < \epsilon.
 \end{align*}
\end{theorem}

From an applied perspective, normalized encoders $\gNorm_i : \XX_i \to
 \Simplex^{\dLat - 1}$ offer a natural geometric interpretation via the
cosine similarity, and recent work provides additional theoretical
justification for their use in contrastive learning
\citep{wangUnderstandingContrastiveRepresentation2022,
 fengWhenDoesEmbedding2026}. We now extend the preceding result to this
setting.

\begin{theorem}[\eqref{eq:ClipTiltingFunctionBiModal} UA, Normalized Encoders]\leavevmode
 \label{thm:NormalizedTwoClipUniversalApproximation}
 Consider measure $\mu$ satisfying \Assref{ass:ProblemSettingDataDistribution}.
 Assume that $\XX_1 \Subset \RR^{d_1}, \XX_2 \Subset \RR^{d_2}$ are compact, and
 let $\MM^\nn$ be a parametric family satisfying
 \UAAssref{ass:DensityParametricFamilyInContinuousFunctions}.
 Then, for every $\epsilon > 0$, there exist a latent dimension $\dLat \in
  \NN$, $p\in\NN$, and parameters $\thetaVecHat = (\thetaHat_1, \thetaHat_2) \in \Theta \subseteq \RR^p$
 with corresponding \emph{normalized encoders}
 $\gNorm_i(\innerEmpty{};\thetaHat_i) : \XX_i \to \Simplex^{\dLat - 1}$,
 $i=1,2$, and a temperature $\tau > 0$ defining
 \eqref{eq:ClipTiltingFunctionBiModal} such that the measure induced by
 \eqref{eq:GeneralTiltingFunction} fulfills
 \begin{align*}
  \KL{\mu}{\mu^{\twoclip}(\innerEmpty{};\thetaVecHat)} < \epsilon.
 \end{align*}
\end{theorem}
\begin{proof}
 (Sketch Proof of Theorems \ref{lem:DensityClipUnnormalizedEncoders}
 and \ref{thm:NormalizedTwoClipUniversalApproximation})
 Full details of the proofs are contained in
 \Cref{lem:DensityClipUnnormalizedEncoders-PROOF} and
 \Cref{thm:NormalizedTwoClipUniversalApproximation-PROOF}, respectively.
 The key idea in both is to use a Stone-Weierstraß argument, to
 show that the parametric family
 \begin{align*}
  \MM^{\twoclip} \define  \{\xx \mapsto
   & \IP{\gNorm_1(x_1;\theta_1)}{\gNorm_2(x_2;\theta_2)}
  / \tau \;\big\vert\; \tau > 0, p\in \NN,                      \\
   & \qquad \dLat \in \NN, \theta_i \in \Theta \subseteq \RR^p,
  \gNorm_i(\innerEmpty{};\theta_i): \XX_i
  \to \Simplex^{\dLat -1}\}
 \end{align*}
 is dense in $C^0(\XX;\RR)$. The result then follows by
 \Cref{result:UniversalApproxJointNetwork} and
 \Cref{result:UniversalApproxJointNetworkTightness}, respectively.
 The key difficulty in the normalized case is
 that the normalization constraint $\gNorm_i(\XX_i) \subseteq
  \Simplex^{\dLat - 1}$ restricts the inner product to
 $\IP{\gNorm_1}{\gNorm_2} \in [-1, 1]$, which in turn confines the tilting
 function to $[e^{-1/\tau}, e^{1/\tau}]$. Making $\tau$ learnable resolves
 this: choosing $\tau$ sufficiently small restores the full range and
 ensures universal approximation. The construction combines orthogonal
 padding with a temperature rescaling, yielding a latent dimension equal to
 the Stone-Weierstraß rank plus one additional padding dimension.
\end{proof}

The theorems above show that in the classical (bi-modal) CLIP setting, the
expressiveness of the \eqref{eq:ClipTiltingFunctionBiModal} function is as
powerful as that of the more general joint network $\tilt^\nn$. This
provides theoretical understanding of recent empirical successes
demonstrated in, for example, \citep{jiaScalingVisualVisionLanguage2021},
in which it is shown that the simple two-tower CLIP architecture yields
competitive results compared to more sophisticated methods, provided that
enough (high-quality) data is available.

In practice, one hopes to reduce the dimensionality of the latent space,
i.e. the encoders $\gNorm_i$'s should compress the data. In this context we
emphasize that neither \Cref{lem:DensityClipUnnormalizedEncoders} nor
\Cref{thm:NormalizedTwoClipUniversalApproximation} give any guarantees on
the dimensionality of the latent space. In fact, the following proposition
shows that CLIP does not in general guarantee a latent representation below
the dimension of the original joint domain.

\begin{proposition}[No Guaranteed Dimensionality Reduction for
  CLIP]\leavevmode
 \label[proposition]{eq:NegativeExampleDimensionalityLatentSpace}
 There exist $D\in\NN$, a compact bimodal domain
 $\XX=\XX_1\times\XX_2\subset\RR^D$ with nonempty interior, a probability
 measure $\mu$ on $\XX$, and a constant $\kappa>0$ such that, for every
 $p\in\NN$, every parameter space $\Theta\subseteq\RR^p$ and associated pair
  of CLIP encoder parameterizations, every $\thetaVec\in\Theta$, every
  $\tau>0$, and every $\dLat<D$,
 \begin{align*}
  \KL{\mu}{\mu^{\twoclip}(\innerEmpty{};\thetaVec)}
  \geq\kappa.
 \end{align*}
 Thus, neither the parameter dimension nor the encoder architecture is fixed.
 The encoders need not be normalized, so the same lower bound also holds when
 they are restricted to be normalized.
\end{proposition}
\begin{proof}
 (Sketch, Full details in \Cref{eq:NegativeExampleDimensionalityLatentSpace-PROOF})
 Take $\XX_1=\XX_2=[-1,1]$, so that $D=2$, equip both
 modalities with uniform marginals, and let $\mu$ be induced by the tilting
 function
 \begin{align*}
  \tilt_0(x_1,x_2)
  =\cos\bigl(\pi(x_1-x_2)\bigr)
  +\cos\bigl(2\pi(x_1-x_2)\bigr).
 \end{align*}
 Using $\tilt_0$ as a kernel defines a Hilbert--Schmidt integral operator
  $K_{\tilt_0}$ on the common marginal $L^2$-space. The four Fourier
  functions $\cos(\pi\innerEmpty)$, $\sin(\pi\innerEmpty)$,
  $\cos(2\pi\innerEmpty)$, and $\sin(2\pi\innerEmpty)$ are mutually orthogonal
  eigenfunctions with eigenvalue $\nicefrac{1}{2}$, so $K_{\tilt_0}$ has rank four. 
  A localized finite-dimensional Eckart--Young--Mirsky argument followed by Pinsker's inequality gives a constant
 $\kappa>0$, uniformly over all temperatures, parameter dimensions, and
 encoder parameterizations, such that
 \begin{align*}
  \KL{\mu}{\mu^{\twoclip}(\innerEmpty{};\thetaVec)}
  \geq\kappa,
  \qquad \dLat=d\in\{1,2,3\}.
 \end{align*}
 In particular, the case $d=1<D$ proves the claim.
\end{proof}

\begin{remark}[Dictionary-Based Encoders]
 \label[remark]{rem:DictionaryBasedEncoders}
 The proof of \Cref{eq:NegativeExampleDimensionalityLatentSpace} suggests
  constructing encoders from a structured dictionary of the form:
  \begin{align*}
   \tilt^\app(x_1,x_2)
   =\frac{1}{\tau}\sum_{k=1}^{\dLat}\phi_k(x_1)\psi_k(x_2)
   =\frac{1}{\tau}\IP{g_1(x_1)}{g_2(x_2)}.
  \end{align*}
  The dictionary may be selected using domain knowledge, for example Fourier
  modes for periodic structure or symmetry-adapted features for known
  invariances. In the construction above, four Fourier features yield an exact
  representation with normalized encoders and $\tau=\nicefrac{1}{2}$, so the
  latent-dimension threshold is sharp.
\end{remark}

\section{Contrastive Learning $m>2$}
\label{sec:ContrastiveLearningBeyondBimodal}
In this section, we investigate universal approximation beyond the two
modality setting studied in the previous section. We show that the
multimodal generalization of CLIP \eqref{eq:ClipTiltingFunctionGeneral}
commonly used in practice fails to universally approximate the joint
distribution. The obstruction is structural:
\eqref{eq:ClipTiltingFunctionGeneral} only captures pairwise
interactions between modalities. This makes it sufficiently expressive to
universally approximate pairwise conditionals, as shown in
\Cref{thm:UniversalApproxClipPairwiseConditionals}; but it cannot
universally approximate the full joint,
as \Cref{thm:noUniversalApproximationClip} demonstrates.
We thus propose a simple tensor-decomposition
inspired extension of multimodal CLIP, defined in
\eqref{eq:HadamardClipTiltingFunction}, that recovers the full
expressiveness of CLIP at negligible additional computational overhead;
see \Cref{thm:UniversalApproximationHadamardClip}.

\begin{theorem}[\eqref{eq:ClipTiltingFunctionGeneral} No UA of the
  Joint]
 \label{thm:noUniversalApproximationClip}
 For every $m\geq3$, there exist a measure $\mu$ satisfying
 \Assref{ass:ProblemSettingDataDistribution} and $\kappa>0$ such that, for all
  $\dLat$, $p\in\NN$, all parameter spaces $\Theta\subseteq\RR^p$ and associated
  parameterizations of $\MM^{\mclip}$, all $\thetaVec\in\Theta$, and all
  $\tau>0$,
 \begin{align*}
  \Expu{x_{-3} \sim \mu_{-3}}{\KL{\mu_{x_3\vert x_{-3}}}{\mu^{\mclip}_{
  x_3\vert x_{-3}}(\innerEmpty{};\thetaVec)}} \geq\kappa.
 \end{align*}
 Consequently, \Cref{thm:JointVsFullConditionals} gives
 \begin{align*}
  \KL{\mu}{\mu^{\mclip}(\innerEmpty{};\thetaVec)} \geq\kappa.
 \end{align*}
 Thus, even after taking the union over all latent and parameter dimensions
  and all encoder parameterizations, the resulting class of measures is not
  rich enough to universally approximate the joint distribution of
  arbitrary data-generating measures, in our setting.
\end{theorem}
\begin{proof}
 (Sketch, Full details in \Cref{thm:noUniversalApproximationClip-PROOF})
 We construct a measure $\mu$ whose density with respect to
 $\mu_\prodMeas$ contains a three-way interaction $x_1x_2x_3$ and show that
 \eqref{eq:ClipTiltingFunctionGeneral} cannot approximate the full
 conditional; together with \Cref{thm:JointVsFullConditionals} this yields
 the desired result.
\end{proof}

The theorem shows that the generalization
\eqref{eq:ClipTiltingFunctionGeneral} of
\eqref{eq:ClipTiltingFunctionBiModal} to $m \geq 3$ modalities is
insufficient to universally approximate arbitrary distributions $\mu$,
since the separable structure restricts the model to pairwise interactions,
precluding higher-order multi-modal effects. Interestingly, this
restriction is harmless for learning pairwise conditionals. In the setting
of \Cref{thm:JointVsPairwiseConditionals}, the architecture can universally
approximate all pairwise conditionals simultaneously. As before, we focus
on compact support, with extensions to non-compact $\XX$ available via
tightness.
\begin{theorem}[\eqref{eq:ClipTiltingFunctionGeneral}
  UA of Pairwise Conditionals]
 \label{thm:UniversalApproxClipPairwiseConditionals}
 Consider measure $\mu$ satisfying \Assref{ass:ProblemSettingDataDistribution}.
 Assume that $\XX \Subset \RR^{d}$ is compact, and let $\MM^\nn$ be a parametric
 family satisfying
 \UAAssref{ass:DensityParametricFamilyInContinuousFunctions}.
 Then for every $\epsilon > 0$, there exists a latent dimension $\dLat \in
  \NN$, $p \in \NN$, parameters $\thetaVecHat = (\thetaHat_1, \ldots, \thetaHat_m) \in
  \Theta \subseteq \RR^p$ with corresponding normalized encoders
 $\gNorm_{i}(\innerEmpty{};\thetaHat_i) : \XX_i \to \Simplex^{\dLat -1}$ and
 a temperature $\tau > 0$ defining \eqref{eq:ClipTiltingFunctionGeneral}
 such that the measure $\mu^{\mclip}(\innerEmpty{};\thetaVecHat)$ induced by
 \eqref{eq:GeneralTiltingFunction} universally approximates \emph{all
  pairwise} conditionals
 \begin{align*}
  \Expu{x_j \sim \mu_j}{\KL{\mu_{x_i \vert x_j}}{
    \mu^{\mclip}_{x_i\vert x_j}(\innerEmpty{};\thetaVecHat)}} < \epsilon,
 \end{align*}
 for pairings $(i,j) \in \setN{m}^2$, $i\neq j$.
\end{theorem}
\begin{proof} (Sketch, Full details in
 \Cref{thm:UniversalApproxClipPairwiseConditionals-PROOF}) We formulate the problem of approximating pairwise
 conditional distributions as a convex constrained projection problem.
 By a result of \citep{csiszar$I$DivergenceGeometryProbability1975},
 we obtain a jointly consistent $L^1$ representative of all pairwise
 conditionals, which takes the form of a sum of two-variable functions
 within an exponential family. We then universally approximate each
 representative using the results of \Cref{sec:CLIP}, thereby approximating
 each pairwise conditional. Finally, using a block structure with additional
 padding dimensions to satisfy the encoders' normalization constraints,
 we combine the individual encoders consistently into the
 architecture of \eqref{eq:ClipTiltingFunctionGeneral}.
\end{proof}

Practitioners have addressed empirically observed limitations of CLIP, in
particular for $m \ge 3$ modalities, via \emph{late fusion} or
\emph{two-leg architectures} \citep{yuanSurveyMultimodalLearning2025}. In
these architectures, a small fusion network
$g^\fusion(\innerEmpty{};\omega)$ acts on top of frozen or jointly trained
encoders. Formally, the tilting function takes the compositional form
$$\tilt^\fusion(\xx; \thetaVec, \omega) \define g^{\fusion}(\gNorm(x_1;
 \theta_1), \ldots, \gNorm(x_m;\theta_m); \omega),$$ where $g^\fusion$ is
e.g. another neural network parametrized by parameter $\omega$. We ask:
what is the \emph{simplest fusion architecture} that provably
\emph{restores universal approximation}?

Classical m-CLIP fails because its tilting function captures only pairwise
interactions between modalities, as formalized in
\Cref{thm:noUniversalApproximationClip} above. To restore universal
approximation, a fusion architecture must be capable of generating
degree-$m$ multilinear interactions. Motivated by tensor decompositions
\citep{ballardTensorDecompositionsData2025}, we propose a particularly
simple fusion: a weighted sum of element-wise products of the encoder
outputs. This creates a structure equivalent to a rank-$\dLat$ Tucker
decomposition in the latent space corresponding to a tilting function
$\tilt^\app = \tilt^\hadamard$ of the form
\begin{align}
 \begin{aligned}
  \tilt^\hadamard(x_1, \ldots, x_m;\thetaVec, \omega)
   & \define \sum_{k=1}^{\dLat} \omega_k
  \prod_{i=1}^{m} (\gNorm_i(x_i;\theta_i))_k \\
   & = \omega^\top
  \left(\gNorm_1(x_1;\theta_1) \odot \cdots \odot \gNorm_m(x_m;\theta_m)\right).
 \end{aligned}
 \tag{\ensuremath{\tilt^\hadamard}-Tilting}
 \label{eq:HadamardClipTiltingFunction}
\end{align}
In the notation of \eqref{eq:GeneralTiltingFunction}, corresponding model
measure is $\mu^{\hadamard}(\innerEmpty{};\thetaVec,\omega) \define
 \mu^{\tilt^\hadamard(\innerEmpty{};\thetaVec,\omega)}$.
We call this architecture Hadamard-CLIP (H-CLIP) because its fusion layer
 applies a learned linear functional to the \emph{Hadamard product} of all
 \emph{encoder outputs}.
This adds only a single learnable linear
layer $\omega \in \mathbb{R}^{\dLat}$ on top of the existing encoders,
yet suffices to guarantee universal approximation as the following
theorem shows. As before, this result can be generalized to
$\XX$ non-compact via tail-truncation.

\begin{theorem}[Hadamard-CLIP, Universal Approximation]
 \label{thm:UniversalApproximationHadamardClip}
 Consider measure $\mu$ satisfying \Assref{ass:ProblemSettingDataDistribution}.
 Assume that $\XX \Subset \RR^{d}$ is compact, and let $\MM^\nn$ be a parametric
 family satisfying
 \UAAssref{ass:DensityParametricFamilyInContinuousFunctions}.
 Then, for every $\epsilon > 0$, there exist a latent dimension $\dLat \in
  \NN$, $p\in \NN$, as well as parameters $\thetaVecHat \in \Theta \subseteq \RR^p$ with corresponding
 \emph{normalized} encoders $\gNorm_i(\innerEmpty{};\thetaHat_i) : \XX_i \to
  \Simplex^{\dLat - 1}$, $i=1,\ldots, m$, and a weight vector $\omegaHat \in
  \RR^{\dLat}$ defining \eqref{eq:HadamardClipTiltingFunction} such that the
 measure $\mu^{\hadamard}(\innerEmpty{};\thetaHat, \omegaHat)$ induced by
 \eqref{eq:GeneralTiltingFunction} satisfies
 \begin{align*}
  \KL{\mu}{\mu^{\hadamard}(\innerEmpty{};\thetaVecHat, \omegaHat)} < \epsilon.
 \end{align*}
\end{theorem}
\begin{proof}(Sketch, Full details in \Cref{thm:UniversalApproximationHadamardClip-PROOF})
 Using a Stone-Weierstraß argument, we show that the parametric family
 \begin{align*}
  \MM^{\hadamard} = \bigcup_{p \in \NN} \bigg\{
  \xx \mapsto & \sum_{k=1}^{\dLat} \omega_k \prod_{i=1}^{m} \gNorm_{i,k}(x_i;\theta_i)
  = \omega^\top (\gNorm_1(x_1; \theta_1) \odot \cdots \odot \gNorm_m(x_m; \theta_m))
  \; \big\vert\;                                                                       \\
              & \qquad\dLat\in \NN, \omega \in \RR^{\dLat},
  \{\theta_i\}_{i=1}^m \subset \Theta
  \subset \RR^p,
  \gNorm_i(\innerEmpty{};\theta_i) : \XX_i \to \Simplex^{\dLat - 1}
  \bigg\}
 \end{align*}
 is dense in $C^0(\XX; \RR)$. The result is then a consequence of
 \Cref{prop:FuncToMeasCompactSupport}.
\end{proof}
A key practical advantage of \eqref{eq:ClipTiltingFunctionGeneral} is
that \emph{retrieval} and \emph{classification at inference}
require only stored embeddings, not the raw data
\cite[Sec.~4]{baptistaMathematicalPerspectiveContrastive2025}.
While using a general joint network ensures universal approximation
by \Cref{result:UniversalApproxJointNetwork}, it does not permit
this efficiency.

The structure of \eqref{eq:HadamardClipTiltingFunction} preserves this
property. For a fixed modality $i_0 \in \setN{m}$, we define the
\emph{pre-computable embeddings} as a Hadamard product of the form
\begin{align*}
 v^{(n)} \define \bigodot_{i \neq i_0} \gNorm_i(x_i^{(n)};\theta_i)
 \in \RR^{\dLat}, \quad n = 1, \ldots, N.
\end{align*}
Given a query vector $x_{i_0} \in \XX_{i_0}$, retrieval of the best matching
database item then reduces to
\begin{align*}
 x_{-i_0}^*
 \in \argmax_{n=1, \ldots, N}\; \omega^\top \left(
 v^{(n)} \odot \gNorm_{i_0}(x_{i_0};\theta_{i_0})
 \right)
 = \argmax_{n=1, \ldots, N}\;
 \bigl(\omega \odot \gNorm_{i_0}(x_{i_0};\theta_{i_0})\bigr)^\top v^{(n)},
\end{align*}
where the equality follows from $\omega^\top(a \odot b) = (\omega \odot
 a)^\top b$. Inference therefore requires a single forward pass through
$g_{i_0}$, followed by $N$ dot products against the stored
$\dLat$-dimensional vectors $v^{(n)}$ which drastically reduces storage
requirements in large-scale applications.
\begin{remark}[Geometric Interpretation of \eqref{eq:HadamardClipTiltingFunction}]
 We provide geometric intuition for\\ \eqref{eq:HadamardClipTiltingFunction}
 and relate it to \eqref{eq:ClipTiltingFunctionGeneral}. Denote with
 $e_k \in \RR^{\dLat}$
 the k-th standard basis vector.

 \textbf{CLIP.} The CLIP \eqref{eq:ClipTiltingFunctionGeneral}
 measures \emph{pairwise} feature-by-feature agreement via
 \begin{align*}
  \IP{\gNorm_i}{\gNorm_j} = \sum_{k=1}^{\dLat}
  \IP{\gNorm_i}{e_k} \IP{\gNorm_j}{e_k},
 \end{align*}
 multiplying the projections of each modality onto axis $e_k$
 and summing across axes. The per-axis term is positive if and only if both modalities
 project onto $e_k$ with the same sign, so the score is large when the two
 modalities agree axis-by-axis on the latent dimensions
 $\{e_k\}_{k=1}^{\dLat}$.

 \textbf{H-CLIP.} The Hadamard
 \eqref{eq:HadamardClipTiltingFunction} replaces the pairwise product with
 the $m$-fold product
 \begin{align*}
  \omega^\top (\gNorm_1 \odot \cdots \odot \gNorm_m)
  = \sum_{k=1}^{\dLat} \omega_k \prod_{i=1}^m \IP{\gNorm_i}{e_k},
 \end{align*}
 the natural $m$-way generalization of axis-wise agreement. The geometric
 reading shifts in two ways. First, the unweighted product
 $\prod_{i=1}^{m}\IP{\gNorm_i}{e_k}$
 is large in magnitude only when
 \emph{every modality} has non-trivial loading on $e_k$, so each axis acts
 as a latent concept that contributes to the score only under joint
 participation across all $m$ modalities; a single small projection
 suppresses the term regardless of the remaining $m-1$ loadings. Second,
 the learnable weights $\omega\in \RR^{\dLat}$, which in the bimodal
 case are redundant and can be absorbed into the encoder rescaling, become
 essential for $m\geq 3$, since under the normalization
 $\gNorm_i \in \Simplex^{\dLat -1}$ the signs of $\omega_k$ (for odd $m$)
 and the  per-axis magnitudes can no longer be folded into the encoders without
 additional padding dimensions.
\end{remark}
\section{Conclusion and Future Work}
The perspective in this paper is to view contrastive learning as a
methodology for learning approximations of joint probability distributions
known to us only through data. We adopt a population-level perspective and
focus on expressivity, asking which joint distributions a given contrastive
parameterization can represent at the population level. We show that
bimodal two-tower CLIP is as expressive as a general joint network; that
the sum-of-pairs generalization used in practice provably fails to
represent arbitrary joints for three or more modalities while still
matching every pairwise conditional; and that this gap is closed by
Hadamard-CLIP, a minimal modification adding a single learned weight vector
that restores universal approximation for any number of modalities while
preserving fast, precomputable-embedding retrieval.

Our analysis leaves several questions open. First, our guarantees do not
provide sharp bounds on the latent dimension required for a given
approximation accuracy; \Cref{eq:NegativeExampleDimensionalityLatentSpace}
shows that compression below the original dimension is in general
impossible, but a tight characterization of the dimension-accuracy tradeoff
remains open. Second, we work entirely in the population limit, opening the
study of statistical learning to quantify the effect of empirical
approximation. Thirdly we isolate expressivity from optimization: whether
the contrastive objective actually recovers these expressive solutions, for
instance through a mean-field analysis of the training dynamics,
constitutes a challenging research question. Fourthly the use of
divergences other than KL, such as those based on the energy score
\citep{waghmareProperScoringRules2026}, could lead to improved
computational methodology, and to interesting theoretical questions
analogous to those studied, or proposed, in this paper.

\section*{Acknowledgement}
AMS is supported by a Department of Defense (DoD) Vannevar Bush Faculty
Fellowship (award N00014-22-1-2790). FW is supported by a Kortschak
Fellowship at The California Institute of Technology.
The authors would like to thank Zakaria Baba, Ricardo Baptista,
Georgia Gkioxari, Raphaela Kang, and Pietro Perona
for fruitful discussions.
\bibliography{./contrastive_learning_zotero_sync,
 ./bib_add_bib_entries_here_contrastive_learning}
\appendix
\section{Technical Lemmas and Proofs}
This appendix collects the technical results and full proofs underlying
 the main text. The first subsection proves that joint KL divergence controls
 the full and pairwise conditional losses. The second develops the general
 function-to-measure approximation tools used throughout our universal
 approximation arguments. The third establishes the bimodal CLIP results,
 including the normalized-encoder construction and the latent-dimension
 obstruction. The fourth proves the failure of joint universal approximation
 for the standard extension to $m>2$ and constructs the consistent pairwise
 representative used thereafter. The fifth proves universal approximation of
 all pairwise conditionals, and the final subsection proves universal
 approximation for Hadamard-CLIP.

\subsection{Contrastive Learning Population-Level Loss Functions}
This subsection provides the full proofs of
 \Cref{thm:JointVsFullConditionals,thm:JointVsPairwiseConditionals}, which
 justify the population-level loss comparisons used in
 \Cref{subsec:SettingMultimodalContrastiveLearning}.

\subsubsection{Proof of \Cref{thm:JointVsFullConditionals}}
\label{thm:JointVsFullConditionals-PROOF}
\begin{proof}
 First, we verify that all KL divergences and conditional distributions
 appearing in the two statements are well-defined.

 \emph{(1) Absolute Continuity Relations.} As highlighted before
  \eqref{eq:RadonNikodymChainRule}, $\tilt^\app$ is finite-valued, and hence
  $\Dd\mu^\app/\Dd\mu_\prodMeas>0$ on $\XX$. Since
  $\mu\ll\mu_\prodMeas$, this yields $\mu\ll\mu^\app$, so the
  Radon-Nikodym chain rule \eqref{eq:RadonNikodymChainRule} applies even when
  $\tilt$ takes values in $\RR\cup\{\pm\infty\}$. For the marginals, by
 Assumption \ref{ass1:MarginalDistribution} every marginal
 $\mu_i$ is absolutely continuous with respect to $\lambda\big\vert_{\XX_i}$ for
 every $i \in \setN{m}$. Since $\mu^\app \ll \mu_\prodMeas$, by marginalizing
 it holds that every marginal $\mu_i$ is absolutely continuous with respect to $\mu_i^\app$
 (interpreted as a measure on $\XX_i$). Analogously, $\mu_{-i} \ll \mu_{-i}^\app$
 and $\mu_{ij} \ll \mu_{ij}^\app$, since marginalization preserves absolute continuity.

 By the disintegration theorem, the conditional distributions $\mu_{x_i\vert
   x_{-i}}$ and $\mu^\app_{x_i\vert x_{-i}}$ exist $\mu_{-i}$-almost surely
 and $\mu_{-i}^\app$-almost surely, respectively. The absolute continuity
 $\mu_{x_i\vert x_{-i}} \ll \mu^\app_{x_i\vert x_{-i}}$ follows from $\mu
  \ll \mu^\app$ and the chain rule representation below \cite[Thm.
  2.5.3]{coverELEMENTSINFORMATIONTHEORY}. The same reasoning applies to all
 pairwise conditionals in the proof of
 \Cref{thm:JointVsPairwiseConditionals}.

 \emph{(2) Upper Bound Full Conditionals.} Fix an $i\in \setN{m}$. The KL divergence chain
 rule \cite[Thm. 2.5.3]{coverELEMENTSINFORMATIONTHEORY} applied to the
 decomposition $\mu = \mu_{-i} \otimes \mu_{x_i \vert x_{-i}}$ and for
 $\mu^\app = \mu^\app_{-i} \otimes \mu^\app_{x_i\vert x_{-i}}$ yields
 \begin{align}
  \KL{\mu}{\mu^\app} = \KL{\mu_{-i}}{\mu^\app_{-i}}
  + \Expu{x_{-i}\sim \mu_{-i}}{
   \KL{\mu_{x_i\vert x_{-i}}(\innerEmpty{}\vert x_{-i})}{
    \mu^\app_{x_i\vert x_{-i}}(\innerEmpty{}\vert x_{-i})
   }
  }
  \label{eq:JointGeCondFullCondDecomposition}
 \end{align}
 since $\KL{\mu_{-i}}{\mu^\app_{-i}} \geq 0$ we obtain
 \begin{align*}
  \Expu{x_{-i}\sim \mu_{-i}}{
   \KL{\mu_{x_i\vert x_{-i}}(\innerEmpty{}\vert x_{-i})}{
    \mu^\app_{x_i\vert x_{-i}}(\innerEmpty{}\vert x_{-i})
   }
  } \leq \KL{\mu}{\mu^\app}.
 \end{align*}
\end{proof}
\subsubsection{Proof of \Cref{thm:JointVsPairwiseConditionals}}
\label{thm:JointVsPairwiseConditionals-PROOF}
\begin{proof}
 \emph{(1) Absolute Continuity.} This step is the same as in
 the proof of \Cref{thm:JointVsFullConditionals}, cf.
 \Cref{thm:JointVsFullConditionals-PROOF}, applied to the pairwise
 conditionals.

 \emph{(2) Upper Bound Pairwise Conditionals.} For a pair
 $(i,j) \in \setN{m}^2$, $i\neq j$,
 we apply the KL divergence chain rule
 (see proof of Proof of \Cref{thm:JointVsFullConditionals}) on the decomposition
 $\xx = ((x_i, x_j), x_{\text{-}(i,j)})$ and obtain
 \begin{align*}
  \KL{\mu}{\mu^\app} = \KL{\mu_{ij}}{\mu^\app_{ij}}
  + \Expu{(x_{i}, x_j)\sim \mu_{ij}}{
   \KL{\mu_{x_{\text{-}(i,j)\vert x_{i},x_j}}(\innerEmpty{}\vert x_{i}, x_j)}{
    \mu^\app_{x_{\text{-}(i,j)\vert x_{i},x_j}}(\innerEmpty{}\vert x_{i}, x_j)}
  }
 \end{align*}
 and since the second term is non-negative, we get
 \begin{align*}
  \KL{\mu_{ij}}{\mu^\app_{ij}} \leq \KL{\mu}{\mu^\app}.
 \end{align*}

 We apply \eqref{eq:JointGeCondFullCondDecomposition} to the bivariate
 marginal measure $\mu_{ij}$ and $\mu^\app_{ij}$ on $\XX_i \times \XX_j$ and
 conditioning on the variable $x_j$ gives
 \begin{align*}
  \KL{\mu_{ij}}{\mu^\app_{ij}} = \KL{\mu_j}{\mu^\app_j}
  + \Expu{x_j \sim \mu_j}{\KL{\mu_{x_i\vert x_j}}{\mu^\app_{x_i\vert x_j}}},
 \end{align*}
 and since $ \KL{\mu_j}{\mu^\app_j} \geq 0$, we obtain
 \begin{align*}
  \Expu{x_j \sim \mu_j}{\KL{\mu_{x_i\vert x_j}}{\mu^\app_{x_i\vert x_j}}}
  \leq \KL{\mu_{ij}}{\mu^\app_{ij}} \leq \KL{\mu}{\mu^\app}.
 \end{align*}
\end{proof}
\subsection{General Universal Approximation Results}
This subsection develops the architecture-independent analytic tools
 that turn approximation of tilting functions into approximation of measures,
 supporting \Cref{result:UniversalApproxJointNetwork,%
 result:UniversalApproxJointNetworkTightness} and the architecture-specific
 results that invoke them. Although the main results concern probability
 measures, the auxiliary compact-support results must be stated for finite
 Borel measures: in the non-compact argument we restrict $\mu$ and
 $\mu_\prodMeas$ to a compact truncation set $\KK$, and these restricted
 measures generally have masses $\mu(\KK)$ and $\mu_\prodMeas(\KK)$ smaller
 than one. Keeping the more general formulation allows these masses, in
 particular $\mu(\KK)$, to be tracked explicitly through the approximation and
 normalization bounds.

\begin{lemma}[Continuous Representative]
 \label[lemma]{lem:LusinsTheoremGeneric}
 Let $\XX$ be a compact subset of a Euclidean space,
 $\mu$ a finite Borel measure on $\XX$ with $\mu(\XX) \in (0,\infty)$,
 and $r:\XX \to [0,\infty]$ be the density
 of $\mu$ with respect to a second
 finite Borel measure $\mu_\prodMeas$, satisfying
 $\tilt = \log r \in L^1(\XX;\mu)$. Then, for every $\epsilon>0$,
 there exists a continuous approximation
 $\tilt^\cont \in C^0(\XX; \RR)$ of the tilting
 function $\tilt$ such that
 \begin{align*}
  \norm{\log r - \tilt^\cont}_{L^1(\XX, \mu)} =
  \norm{\tilt - \tilt^\cont}_{L^1(\XX, \mu)}
  \leq \epsilon \quad
  \text{ and } \quad
  \int_{\XX} \exp(\tilt^\cont) \dd \mu_\prodMeas \leq \mu(\XX) + \epsilon.
 \end{align*}
\end{lemma}
\begin{proof}
 Let $\epsilon > 0$  be fix, by assumption $\log r$ is $\mu$-measurable.

 \high{(1) Double-sided Truncation.} %
 We define a constant $M> 0$ (to be specified later), and consider the truncated function
 \begin{align*}
  s_{ M}(\xx) \define \min(\max(\log r(\xx), -M), M) \in
  [-M, M], \quad \xx\in\XX
 \end{align*}
 by point-wise clipping the logarithm of the density,
 so $s_{M}$ is uniformly bounded.

 \high{Integral Term.} Since this is a concatenation of continuous
 functions with a $\mu_\prodMeas$-measurable function,
 $s_{ M}$ is $\mu_\prodMeas$-measurable and thus also $\mu$-measurable,
 since $\mu \ll \mu_\prodMeas$.
 By construction of $s_{ M}$ we can bound
 \begin{align*}
   & \norm{\log(r) - s_{M}}_{L^1(\XX, \mu)}
  = \int_{\XX} \abs{\log r(\xx) - s_{M}(\xx)} \; \mu(\Dd \xx)        \\
   & = \int_{\{\log r < -M\}}  \abs{\log r(\xx) + M} \; \mu(\Dd \xx)
  + \int_{\{\log r \ge M\}}  \abs{\log r(\xx) - M} \; \mu(\Dd \xx)   \\
   & \leq 2\int_{\{\log r < -M\}}  \abs{\log r(\xx)} \; \mu(\Dd \xx)
  + 2\int_{\{\log r \ge M\}}  \abs{\log r(\xx)} \; \mu(\Dd \xx)
 \end{align*}
 Since by assumption, $\log(r) \in L^1(\XX, \mu)$, the dominated convergence
 theorem yields $\norm{\log r - s_{ M}}_{L^1(\XX, \mu)} \to 0$ as $M\to
  \infty$ and thus, we can pick $M = M(\epsilon)> 0$ large enough such that
 in total $\norm{\log(r) - s_{M}}_{L^1(\XX, \mu)} < \frac{\epsilon}{2}$.

 \high{Normalization Constant.} We define $r_M \define \exp(s_M)
  \in [\exp(-M), \exp(M)]$ and we split the normalization term into:
 \begin{align*}
  \int_{\XX} r_m \dd \mu_\prodMeas =
   & \underbrace{\int_{\{\exp(-M)\leq r \leq \exp(M)\}} r \dd
  \mu_\prodMeas}_{\uparrow \mu(\XX) \text{ as } M\to \infty}                     \\
   & + \underbrace{\exp(-M)\cdot\mu_\prodMeas(\{r
  < \exp(-M)\})}_{\leq \exp(-M)\mu_\prodMeas(\XX) \to 0 \text{ as } M\to \infty} \\
   & + \underbrace{\exp(M)\cdot\mu_\prodMeas(\{r
   > \exp(M)\})}_{\leq \int_{\{r > \exp(M)\}} r \dd \mu_\prodMeas \to 0 \text{ as }
  M\to \infty},
 \end{align*}
 where the first term convergence by monotone convergence to $\mu(\XX)$
 as $M\to \infty$. The second term converges to zero as $M\to \infty$
 since $\mu_\prodMeas(\XX)$ is finite and $\exp(-M) \to 0$ and the third
 term converges to zero as $M\to \infty$ by dominated convergence.

 Importantly, we can pick $M = M(\epsilon) > 0$ large enough such that
 $\norm{\log(r) - s_{M}}_{L^1(\XX, \mu)} < \frac{\epsilon}{2}$ and
 $\abs{\int_{\XX} r_m \dd \mu_\prodMeas - \mu(\XX)} \leq
  \frac{\epsilon}{2}$.

 \high{(2) Lusin's Theorem on an Auxilary Measure.}
 Set $\eta\define \min\{\frac{\eta}{4M}, \frac{\epsilon}{4\exp(M)}\} > 0$.
 We define the measure
 \begin{align*}
  \nu \define \mu + \mu_{\prodMeas},
 \end{align*}
 which is by assumption a finite Borel measure on $\XX$ and $\nu \ll \mu_\prodMeas$.
 Since $s_{M}$ is
 $\mu_\prodMeas$-measurable and by absolute continuity
 $\mu$-measurable, it is also $\nu$-measurable. By Lusin's theorem
 \cite[Thm. 1.14]{evansMeasureTheoryFine2015} applied to $\nu$ and $s_{M}$
 there exists a compact $\KK = \KK(\eta) \Subset \XX$ such that
 $s_{M}\big\vert_{\KK}$ is continuous
 and $\nu(\XX\setminus\KK) < \eta$. By the construction of
 $\nu$ we obtain
 \begin{align*}
  \mu(\XX\setminus\KK), \mu_\prodMeas(\XX\setminus\KK)
   & \leq \nu(\XX\setminus\KK)  < \eta.
 \end{align*}
 By Tietze's extension theorem
 \cite[Cor. 2]{mcshaneExtensionRangeFunctions1934}, we can extend
 $s_{ M}\big\vert_{\KK}$ to a continuous function $\tilt^\cont \in C^0(\XX;\RR)$
 on the entire domain $\XX$, with $\tilt^\cont$ satisfying
 \begin{align*}
  \tilt^\cont(\xx)       & = s_{ M}(\xx) \quad \forall \xx \in \KK \\
  \abs{\tilt^\cont(\xx)} & \leq M \quad \forall \xx \in \XX,
 \end{align*}
 where the global bounds hold because Tietze's theorem preserves the supremum of
 $s_{M}$ which takes values in $[-M, M]$ on $\KK$ by construction.

 \high{(3) $L^1$-error of Lusin Approximation.} Since $\tilt^\cont = s_{M}$ on $\KK$ we
 get
 \begin{align*}
  \norm{s_{ M} - \tilt^\cont}_{L^1(\XX, \mu)} =
  \int_{\XX\setminus\KK} \abs{s_M(\xx) - \tilt^\cont(\xx)} \; \mu(\Dd \xx)
  \leq 2M \cdot \mu(\XX\setminus \KK)
  \leq \frac{\epsilon}{2}.
 \end{align*}
 We combine the two inequalities from above using the triangle inequality to obtain
 \begin{align*}
  \norm{\log r - \tilt^\cont}_{L^1(\XX, \mu)}
  \leq \norm{\log(r) - s_{ M}}_{L^1(\XX, \mu)}
  + \norm{s_{ M} - \tilt^\cont}_{L^1(\XX, \mu)} \leq \epsilon,
 \end{align*}
 and for the normalization constant, we have
 \begin{align*}
  \abs{\int_{\XX} \exp(\tilt^\cont) \dd \mu_\prodMeas - \int_{\XX} r_M \dd \mu_\prodMeas}
  \leq 2 \exp(M) \cdot \mu_\prodMeas(\XX\setminus\KK)
  \leq 2 \exp(M) \eta < \frac{\epsilon}{2}.
 \end{align*}
 Combined with the second part of step (1), we obtain
 \begin{align*}
  \int_{\XX} \exp(\tilt^\cont) \dd \mu_\prodMeas \leq \mu(\XX) + \epsilon,
 \end{align*}
 concluding the claim.
\end{proof}

We formulate the following meta-theorem for finite Borel measures, which we
will later on apply directly to $\XX$ in the compact support setting, and
in the non-compact setting to the compact subset $\KK \Subset \XX$ given by
tightness of $\mu$ and $\mu_\prodMeas$. 

\begin{proposition}[Function-to-Measure, Compact Support]
 \label[proposition]{prop:FuncToMeasCompactSupport}
 Let $\XX$ be a compact subset of a Euclidean space,
 $\mu_\prodMeas$ a finite Borel measure and
 $r: \XX \to [0,\infty]$ the density of $\mu$ with respect to $\mu_\prodMeas$,
 satisfying $\tilt = \log r \in L^1(\XX, \mu)$.

 Let $\tilt^\cont$ be the continuous function from
 \Cref{lem:LusinsTheoremGeneric} with parameter $\epsilon >0$. For any
 measurable function
 $\tilt^\app : \XX \to \RR$ satisfying
 \begin{align*}
  \norm{\tilt^\cont - \tilt^\app}_{\infty, \XX} \leq \delta
 \end{align*}
 for some $\delta > 0$, the measure $\mu^\app$ constructed via
 \eqref{eq:GeneralTiltingFunction} satisfies
 \begin{align*}
  \KL{\mu}{\mu^\app} \leq \epsilon + \delta \cdot \mu(\XX) + \delta
  + \log(\mu(\XX)) + \frac{\epsilon}{\mu(\XX)}.
 \end{align*}
\end{proposition}

\begin{proof}
 We expand the KL divergence using the \eqref{eq:RadonNikodymChainRule} to obtain
 \begin{align*}
  \KL{\mu}{\mu^\app} = \int_{\XX} (\tilt - \tilt^\app) \dd \mu + \log Z^\app.
 \end{align*}
 We can bound the first term by using the triangle inequality
 \begin{align*}
  \int_{\XX} (\tilt - \tilt^\app) \dd \mu \leq \int_{\XX} \abs{\tilt - \tilt^\cont} \dd \mu
  + \int_{\XX} \abs{\tilt^\cont - \tilt^\app} \dd \mu,
 \end{align*}
 where the first term is upper bounded by $\epsilon$ by \Cref{lem:LusinsTheoremGeneric}
 and the second term satisfies
 \begin{align*}
  \int_{\XX} \abs{\tilt^\cont - \tilt^\app} \dd \mu
  \leq \norm{\tilt^\cont - \tilt^\app}_{\infty, \XX} \cdot \mu(\XX)
  \le \delta \cdot \mu(\XX),
 \end{align*}
 bounding the integral part overall by
 \begin{align*}
  \int_{\XX} (\tilt - \tilt^\app) \dd \mu  \leq \epsilon + \delta \cdot \mu(\XX).
 \end{align*}

 Next, we bound the normalization constant $Z^\app$. Since we have
 $\norm{\tilt^\cont - \tilt^\app}_{\infty, \XX} \leq \delta$, we obtain the
 point-wise bound $\tilt^\app(\xx) \leq \tilt^\cont(\xx) + \delta$ for all
 $\xx \in \XX$:
 \begin{align*}
  Z^\app = \int_{\XX} \exp(\tilt^\app) \dd\mu_\prodMeas
  \leq \exp(\delta) \cdot \int_{\XX} \exp(\tilt^\cont) \dd \mu_\prodMeas
  \leq \exp(\delta) \cdot (\mu(\XX) + \epsilon).
 \end{align*}
 Taking the logarithm results in
 \begin{align*}
  \log Z^\app \leq \delta + \log(\mu(\XX) +\epsilon)
  \leq \delta + \log(\mu(\XX)) + \frac{\epsilon}{\mu(\XX)},
 \end{align*}
 which yields the desired result.
\end{proof}

 For measures $\mu$ and $\mu_\prodMeas$ and a tilting function $\tilt$
 as below, fix $\epsilon,\eta>0$ and define the finite Borel measure
 $\nu(A)\define\int_A\abs{\tilt}\dd\mu$. Tightness of $\mu$,
 $\mu_\prodMeas$, and $\nu$ yields a compact $\KK\Subset\XX$ satisfying the
 three tail bounds below. Applying \Cref{lem:LusinsTheoremGeneric} to the
 restricted measures on $\KK$ then yields a continuous function
 $\tilt^\cont$ satisfying the remaining two bounds. Thus, the auxiliary
 objects assumed in the following proposition always exist.

\begin{proposition}[Function-to-Measure, Tightness]
 \label[proposition]{prop:FuncToMeasTightness}
 Let $\XX \subset \RR^d$, not necessarily compact,
 $r: \XX \to [0, \infty]$ be
 the density of the probability measure $\mu$ with respect to
 the probability measure $\mu_\prodMeas$,
 satisfying $\tilt = \log r \in L^1(\XX, \mu)$.

 Let $\epsilon,\eta>0$, and suppose that a compact
  $\KK\Subset\XX$ and $\tilt^\cont\in C^0(\KK;\RR)$ satisfy
 \begin{align*}
   \mu(\XX\setminus\KK) < \eta, \quad \mu_\prodMeas(\XX\setminus\KK) < \eta,
   \quad \mathrm{and} \quad
   \int_{\XX\setminus\KK}\abs{\tilt}\dd\mu < \eta,
  \end{align*}
  as well as the bounds
  \begin{align*}
   \norm{\tilt-\tilt^\cont}_{L^1(\KK,\mu)}\leq\epsilon
   \quad\mathrm{and}\quad
   \int_{\KK}\exp(\tilt^\cont)\dd\mu_\prodMeas
   \leq\mu(\KK)+\epsilon.
  \end{align*}

 Suppose $\tilt^\app : \XX \to \RR$ is any measurable function satisfying
 \begin{align*}
  \norm{\tilt^\cont - \tilt^\app}_{\infty, \KK} \leq \delta
  \quad \text{ and } \quad
  \abs{\tilt^\app(\xx)} \leq B, \; \forall \xx \in \XX\setminus\KK
 \end{align*}
 for constants $\delta > 0$ and $B \geq 0$, then the measure
 constructed via \eqref{eq:GeneralTiltingFunction} fulfills
 \begin{align*}
  \KL{\mu}{\mu^\app} \leq 2\epsilon + 2\delta + (1+B +\exp(B)) \eta.
 \end{align*}
\end{proposition}
\begin{proof}
 We split the proof into two parts: controlling the integral term and
 controlling the normalization constant.

 \high{(1) Integral term.} By applying \eqref{eq:RadonNikodymChainRule} we obtain
 \begin{align*}
  \KL{\mu}{\mu^\app} = \int_{\XX} \tilt - \tilt^\app \dd \mu + \log Z^\app.
 \end{align*}
 We bound the first integral by splitting the domain into $\KK$
 and $\XX\setminus\KK$ to get
 \begin{align*}
  \int_{\XX} \tilt - \tilt^\app \dd \mu =
  \int_{\KK} \tilt - \tilt^\app \dd \mu
  + \int_{\XX\setminus\KK} \tilt - \tilt^\app \dd \mu,
 \end{align*}
 where we can bound the first integral by
 \begin{align*}
  \int_{\KK} \tilt - \tilt^\app \dd \mu
   & = \int_{\KK} \tilt - \tilt^\cont \dd \mu
  + \int_{\KK} \tilt^\cont - \tilt^\app \dd \mu                 \\
   & \leq \norm{\tilt - \tilt^\cont}_{L^1(\KK, \mu)}
  + \delta \cdot \mu(\KK)
  \leq \epsilon + \delta \cdot \mu(\KK)                         \\
   & \leq \epsilon + \delta \cdot \mu(\XX) = \epsilon + \delta,
 \end{align*}
 and we can bound the difference on the tail by
 \begin{align*}
  \int_{\XX\setminus\KK} \tilt - \tilt^\app \dd \mu
   & \leq
  \int_{\XX\setminus\KK} \abs{\tilt - \tilt^\app} \dd \mu
  \leq
  \int_{\XX\setminus\KK} \abs{\tilt} \dd \mu + B \cdot \mu(\XX\setminus\KK) \\
   & \leq \eta + B\eta = (1+B) \eta
 \end{align*}
 where we controlled the first term via $\eta$ and the absolute continuity
 of the integral, and the second term is bounded by
 $\mu(\XX\setminus\KK) < \eta$. In total, the integral term is bounded by
 $\epsilon + \delta + (1+B) \eta$.

 \high{(2) Normalization Constant.} We split the domain again into
 $\KK$ and $\XX\setminus\KK$, use the point-wise bound inside $\KK$
 and the global bound $B$ on the tail:
 \begin{align*}
  \int_{\KK} \exp(\tilt^\app) \dd \mu_\prodMeas
  \leq \exp(\delta) \int_{\KK} \exp(\tilt^\cont) \dd \mu_\prodMeas
  \leq \exp(\delta) \cdot \left(\mu(\KK) + \epsilon\right)
  \leq \exp(\delta) \cdot \left(1 + \epsilon\right).
 \end{align*}
 Outside of $\KK$ we obtain
 \begin{align*}
  \int_{\XX\setminus \KK} \exp(\tilt^\app)\dd\mu_\prodMeas
  \leq \exp(B) \cdot \mu_\prodMeas(\XX\setminus\KK) \leq \exp(B) \eta.
 \end{align*}
 Adding the two gives an upper bound of $Z^\prime$ and taking the logarithm
 results in
 \begin{align*}
  \log Z^\app \leq \delta + \log(1 + \epsilon + \exp(B - \delta)\eta)
  \leq \delta + \epsilon + \exp(B) \eta,
 \end{align*}
 concluding the proof.
\end{proof}

\subsection{CLIP: Contrastive Learning $m=2$}
This subsection supplies the approximation results specific to
 bimodal CLIP: it first treats unnormalized encoders, then transfers the result
 to normalized encoders with a learnable temperature, and finally proves the
 latent-dimension obstruction stated in
 \Cref{eq:NegativeExampleDimensionalityLatentSpace}.

\begin{proposition}[Stone-Weierstraß]
 \label[proposition]{thm:StoneWeierstrass}
 Let $\XX \subset \RR^d$ be a compact set
 and denote with $C(\XX; \RR)$ the Banach-algebra of real-valued
 continuous functions on $\XX$ with the topology induced by the
 supremum norm. Let $A \subset C(\XX; \RR)$ be a sub-algebra
 that separates points and contains
 a non-zero constant function, then $A$ is dense in $C(\XX; \RR)$.
\end{proposition}
\begin{proof}
 See \cite[Thm. VIII.4.7]{wernerFunktionalanalysis2011}.
\end{proof}
We can adapt this result to our particular case. We start with a general
separable composition of continuous functions, which we will then
extend to the case of separable encoders.
We want to highlight that the assumption of $m=2$ modalities is crucial for this
theoretical direction, since with $m\geq 3$,
e.g. for \eqref{eq:ClipTiltingFunctionGeneral},
the closedness under multiplication
in the following construction breaks and we cannot apply the Stone-Weierstraß
theorem.
\begin{lemma}[Stone-Weierstraß, Bimodal]
 \label[lemma]{lem:StoneWeierstrass2DcaseLinDecomposition}
 For $d_1, d_2 \in \NN$ let
 $\XX_1 \subset \RR^{d_1}, \XX_2 \subset \RR^{d_2}$ be compact sets, and let
 $\tilt : \XX_1 \times \XX_2 \to \RR$ be continuous. Then for
 every $\epsilon > 0$, there exist a latent dimension $\dLat \in \NN$,
 continuous functions $g_{1,1}, \dots, g_{1,\dLat} \in C^0(\XX_1;\RR)$,
 and continuous functions $g_{2,1}, \dots, g_{2,\dLat} \in C^0(\XX_2;\RR)$,
 such that
 \begin{align*}
  \norm{\tilt - \sum_{k=1}^\dLat g_{1,k} \cdot g_{2,k}}_\infty < \epsilon.
 \end{align*}
\end{lemma}
\begin{proof}
 We define the set
 \begin{align*}
  A^\twoclip \define \left\{ \sum_{k=1}^\dLat g_{1,k} \cdot g_{2,k}
  \,\big\vert\, \dLat \in \mathbb{N}, \; g_{1,k} \in C^0(\XX_1;\RR),\, g_{2,k}
  \in C^0(\XX_2;\RR) \text{ for } k\in \setN{\dLat}
  \right\}
 \end{align*}
 which is a subalgebra of $C^0(\XX_1 \times \XX_2; \RR)$: it is closed under
 addition, scalar multiplication, and multiplication, and contains the
 constant functions (take $g_{1,i}(x) \equiv 1$, $g_{2,i}(y) \equiv 1$).

 To apply the Stone-Weierstraß theorem, we must verify that $A$ separates
 points of $ \XX_1 \times \XX_2 $. Let $ (x_1, y_1), (x_2, y_2) \in \XX_1
  \times \XX_2 $ with $ (x_1, y_1) \neq (x_2, y_2) $. Then either $ x_1 \neq
  x_2 $ or $ y_1 \neq y_2 $. In the first case, there exists $ g_{1,1} \in
  C^0(\XX_1;\RR) $ such that $g_{1,1}(x_1) \neq g_{1,1}(x_2) $; then the
 function $ g_{1,1}(x) \cdot 1 \in A $ distinguishes the two points. In the
 second case, a similar argument with a function $g_{2,2} \in
  C^0(\XX_2;\RR)$ works.

 Hence $A$ is a subalgebra of $ C^0(\XX_1 \times \XX_2;\RR) $ that separates
 points and contains the constants. By the Stone-Weierstraß theorem, $ A $
 is dense in $ C^0(\XX_1 \times \XX_2;\RR) $ in the supremum norm. Thus, for
 any $\epsilon > 0 $, there exists a function of the form $ \sum_{i=1}^\dLat
  g_{1,k} \cdot g_{2,k} \in A $ such that
 \begin{align*}
  \norm*{\tilt - \sum_{k=1}^\dLat g_{1,k}\cdot g_{2,k}}_\infty < \epsilon,
 \end{align*}
 concluding the proof.
\end{proof}

\subsubsection{Proof of \Cref{lem:DensityClipUnnormalizedEncoders}}
\label{lem:DensityClipUnnormalizedEncoders-PROOF}
\begin{proof}
 We verify that \eqref{eq:ClipTiltingFunctionBiModal} with
 \emph{unnormalized} encoders itself satisfies
 \UAAssref{ass:DensityParametricFamilyInContinuousFunctions}.
 Let $\epsilon > 0$ fixed and without loss of
 generality we assume $\tau = 1$ to be fixed, otherwise we approximate $\tau
  \cdot \tilt$, which is also a continuous function. By
 \Cref{lem:StoneWeierstrass2DcaseLinDecomposition} we can pick $\dLat\in
  \NN$ and continuous functions $g_{i,1}, \ldots, g_{i,\dLat}:\XX_i \to \RR$,
 $i=1,2$, such that
 \begin{align*}
  \norm{\tilt - \sum_{k=1}^{\dLat} g_{1,k}\cdot g_{2,k}}_{\infty} < \frac{\epsilon}{2}.
 \end{align*}
 Define $G \define \max_{i=1,2} \max_{k=1,\ldots, \dLat} \norm{g_{i,k}}_{\infty} < \infty$
 since every $g_{i,k}$ is continuous and $\XX$ is compact.
 With a target accuracy of $\delta \define \min\left\{ \frac{1}{\dLat \cdot (2G + 1)}
  \cdot \frac{\epsilon}{2}, 1\right\}$,
 we invoke our parametric family using \UAAssref{ass:DensityParametricFamilyInContinuousFunctions}
 to obtain parameters $(\thetaHat_{i,k})_{i,k} \in \Theta$ and corresponding
 encoders with
 \begin{align*}
  \norm{g_{i,k} - g_{i,k}(\innerEmpty{};\thetaHat_{i,k})}_{\infty} < \delta,
 \end{align*}
 for $i=1,2$ and $k=1,\ldots, \dLat$. We write
 $\gHat_{i,k} \define g_{i,k}(\innerEmpty{};\thetaHat_{i,k})$ for the sake of brevity
 and obtain via the triangle inequality
 \begin{align*}
  \norm{g_{1,k}g_{2,k} - \gHat_{1,k}\gHat_{2,k}}_{\infty}
   & \leq
  \norm{g_{1,k}g_{2,k} - \gHat_{1,k}g_{2,k}}_{\infty}
  + \norm{\gHat_{1,k}g_{2,k} - \gHat_{1,k}\gHat_{2,k}}_{\infty} \\
   & \leq G \cdot \delta + (G + \delta) \cdot \delta
  \leq (2G+1) \cdot \delta
 \end{align*}
 since $\delta \leq 1$ by construction. For $i=1,2$, we define the (unnormalized) encoders
 \begin{align*}
  g_i(\innerEmpty{};\theta_i) : \XX_i \to \RR^{\dLat};
  \xx \mapsto (g_{i, 1}(\xx, \thetaHat_{i,1}), \ldots g_{i,\dLat}(\xx;\thetaHat_{i,\dLat}))^\top
 \end{align*}
 with parameters $\thetaHat_i \define (\thetaHat_{i,1}, \ldots \thetaHat_{i,\dLat})^\top$.
 By the triangular inequality, we obtain
 \begin{align*}
   & \norm{\tilt - \IP{g_1(\innerEmpty{};\thetaHat_1)}{g_2(\innerEmpty{};\thetaHat_2)}}_{\infty} \\
   & \leq \norm{\tilt - \sum_{k=1}^{\dLat} g_{1,k}\cdot g_{2,k}}_{\infty}
  +\norm{\sum_{k=1}^{\dLat} g_{1,k}\cdot g_{2,k}
  - \IP{g_1(\innerEmpty{};\thetaHat_1)}{g_2(\innerEmpty{};\thetaHat_2)}}_{\infty}                \\
   & \leq \frac{\epsilon}{2} + \dLat \cdot (2G+1) \cdot \delta
  \leq \frac{\epsilon}{2} + \frac{\epsilon}{2} = \epsilon,
 \end{align*}
 verifying \UAAssref{ass:DensityParametricFamilyInContinuousFunctions}.
 Invoking
 \Cref{result:UniversalApproxJointNetwork} and
 \Cref{result:UniversalApproxJointNetworkTightness}, respectively,
 concludes the proof.
\end{proof}

\subsubsection{Proof of \Cref{thm:NormalizedTwoClipUniversalApproximation}}
\label{thm:NormalizedTwoClipUniversalApproximation-PROOF}
\begin{proof}
 We verify that \eqref{eq:ClipTiltingFunctionBiModal} with
 \emph{normalized} encoders itself satisfies
 \UAAssref{ass:DensityParametricFamilyInContinuousFunctions}.
 Let $\epsilon > 0$, then by invoking \Cref{lem:DensityClipUnnormalizedEncoders}
 with $\tau_0 = 1$, we obtain $\dLat \in \NN$, parameters $\thetaVecHat \in \Theta$ and corresponding
 unnormalized encoders $g_i = (g_{i,1}, \ldots, g_{i,\dLat}):\XX_i \to \RR^{\dLat}$,
 $i=1,2$ achieving
 \begin{align*}
  \norm{\tilt - \IP{g_1(\innerEmpty{};\thetaHat_1)}{g_2(\innerEmpty{};\thetaHat_2)}}
  \leq \epsilon.
 \end{align*}
 We define the constant $G$ and a scaling factor $\alpha$ as
 \begin{align*}
  G \define \max_{i=1,2} \max_{k=1,\ldots, \dLat} \norm{g_{i,k}}_{\infty}, \quad
  \alpha \define \frac{1}{2 \sqrt{\dLat} G + 1} > 0,
 \end{align*}
 which guarantee
 $\alpha^2 \sum_{k=1}^{\dLat} g_{i,k}^2(x_i) \leq \alpha^2 \dLat G^2 \leq \frac{1}{4}$
 for both encoders $i=1,2$. We construct \emph{normalized} encoders
 $\gNorm_i :\XX_i \to \Simplex^{\dLat + 1}$ by embedding into orthogonal padding
 dimensions
 \begin{align*}
  \gNorm_1(x_1;\thetaHat_1) & \define (\alpha^2 g_{1,1}(x_1), \ldots, \alpha^2 g_{1,\dLat}(x_1),
  \sqrt{1- \alpha^2 \sum_{k=1}^{\dLat} g_{1,k}(x_1)}, 0)                                         \\
  \gNorm_2(x_2;\thetaHat_2) & \define (\alpha^2 g_{2,1}(x_2), \ldots, \alpha^2 g_{2,\dLat}(x_2),
  0, \sqrt{1- \alpha^2 \sum_{k=1}^{\dLat} g_{2,k}(x_2)}),
 \end{align*}
 which are normalized in the $\norm{\innerEmpty{}}_{2}$-sense by construction.
 Note that the choice of $\alpha$ ensures the well-definedness of the square root
 in the last term. Their inner product is given by
 \begin{align*}
  \IP{\gNorm(x_1;\thetaHat_1)}{\gNorm(x_2;\thetaHat_2)}
  = \alpha^2 \cdot \IP{g_1(x_1;\thetaHat_1)}{g_2(x_2;\thetaHat_2)},
 \end{align*}
 since the padding dimensions are orthogonal and thus do not contribute to the inner
 product. By setting the (learnable) temperature $\tau \define \alpha^2 \tau_0 > 0$,
 we perfectly recover
 \begin{align*}
  \IP{\gNorm(x_1;\thetaHat_1)}{\gNorm(x_2;\thetaHat_2)} / \tau
  = \IP{g_1(x_1;\thetaHat_1)}{g_2(x_2;\thetaHat_2)},
 \end{align*}
 and thus exactly mimic the unnormalized approximation guarantee.
 As in the proof of \Cref{thm:UniversalApproxClipPairwiseConditionals},
 the claim now follows by invoking
 \Cref{result:UniversalApproxJointNetwork} or
 \Cref{result:UniversalApproxJointNetworkTightness}, respectively.
\end{proof}
\subsubsection{Proof of \Cref{eq:NegativeExampleDimensionalityLatentSpace}}
\label{eq:NegativeExampleDimensionalityLatentSpace-PROOF}
\begin{proof}
 We first derive an obstruction to low-rank approximation of the tilting
 function and turn it into a KL lower bound that is uniform over all
 temperatures $\tau>0$.

 Let $\XX_1=\XX_2=[-1,1]$ and $\mu_1=\mu_2=\Uniform([-1,1])$. Thus
 $\mu_\prodMeas=\mu_1\otimes\mu_2$. Consider the similarity score
 \begin{align*}
  \tilt_0(x_1,x_2)
  \define \cos\bigl(\pi(x_1-x_2)\bigr)
  +\cos\bigl(2\pi(x_1-x_2)\bigr).
 \end{align*}
 The normalization constant
 \begin{align*}
  Z \define \int_{-1}^1 \exp\bigl(\tilt_0(x_1,x_2)\bigr)\,\mu_2(\Dd x_2)
 \end{align*}
 is independent of $x_1$ by periodicity. Consequently, the probability
 measure $\mu$ defined by
 \begin{align*}
  \frac{\Dd\mu}{\Dd\mu_\prodMeas}(x_1,x_2)
  =r(x_1,x_2)\define \frac{1}{Z}\exp\bigl(\tilt_0(x_1,x_2)\bigr)
 \end{align*}
 has marginals $\mu_1$ and $\mu_2$. In the notation of
 \eqref{eq:GeneralTiltingFunction}, its tilting function is
 $\tilt=\log r=\tilt_0-\log Z$. Since $\abs{\tilt_0}\leq2$, we have
 \begin{align*}
  \exp(-2)\leq Z\leq\exp(2)
  \qquad\text{and}\qquad
  \exp(-4)\leq r\leq\exp(4).
 \end{align*}
 Expanding the trigonometric differences gives the dictionary representation
 \begin{align}
  \begin{aligned}
   \tilt_0(x_1,x_2)
   ={} & \cos(\pi x_1)\cos(\pi x_2)
   +\sin(\pi x_1)\sin(\pi x_2)         \\
       & +\cos(2\pi x_1)\cos(2\pi x_2)
   +\sin(2\pi x_1)\sin(2\pi x_2).
  \end{aligned}
  \tag{\ensuremath{\tilt_0}-Tilting}
  \label{eq:FourierDictionaryLatentDimension}
 \end{align}
 The four functions appearing in each variable are mutually orthogonal in
 $L^2([-1,1],\Uniform([-1,1]))$ and have squared norm $1/2$.
 Since $\tilt_0\in L^2([-1,1]^2,\mu_\prodMeas)$, it defines, by
 \cite[Thm. VI.6.3]{wernerFunktionalanalysis2011}, a compact
 Hilbert--Schmidt operator
 $K_{\tilt_0}:L^2([-1,1],\mu_2)\to L^2([-1,1],\mu_1)$ via
 \begin{align*}
  (K_{\tilt_0}f)(x_1)
  \define\int_{-1}^1\tilt_0(x_1,x_2)f(x_2)\,\mu_2(\Dd x_2).
 \end{align*}
 Under this isometric identification
  \cite[Thm. VI.6.3]{wernerFunktionalanalysis2011},
 if $K_h$ denotes the integral operator with kernel $h$, then
 \begin{align*}
  \norm{K_h}_{\mathrm{HS}}
  =\norm{h}_{L^2(\mu_\prodMeas)},
 \end{align*}
 and the finite separable rank of $h$ agrees with the operator rank of $K_h$.
 The representation in \eqref{eq:FourierDictionaryLatentDimension} shows
 that $K_{\tilt_0}$ is self-adjoint and positive semidefinite, has rank
 four, and has four non-zero singular values equal to $\nfr{1}{2}$.
 Let
 \begin{align*}
  \mathcal E\define\Span\{
  \cos(\pi\innerEmpty),\sin(\pi\innerEmpty),
  \cos(2\pi\innerEmpty),\sin(2\pi\innerEmpty)\},
 \end{align*}
 and let $P_{\mathcal E}$ denote the orthogonal projection onto $\mathcal E$.
 Then $K_{\tilt_0}=P_{\mathcal E}K_{\tilt_0}P_{\mathcal E}$. For every
 Hilbert--Schmidt operator $T$ with $\rank(T)\leq d$,
 \begin{align*}
  \rank(P_{\mathcal E}TP_{\mathcal E}) & \leq d,                                                           \\
  \norm{K_{\tilt_0}-P_{\mathcal E}TP_{\mathcal E}}_{\mathrm{HS}}
                                       & =\norm{P_{\mathcal E}(K_{\tilt_0}-T)P_{\mathcal E}}_{\mathrm{HS}}
  \leq\norm{K_{\tilt_0}-T}_{\mathrm{HS}}.
 \end{align*}
 Hence the optimization may be restricted to operators on $\mathcal E$.
 With respect to the orthonormalized Fourier basis,
 $K_{\tilt_0}|_{\mathcal E}$ corresponds to the matrix
 $\nfr{1}{2}I_4$, and the Hilbert--Schmidt norm becomes the Frobenius norm.
 The finite-dimensional Eckart--Young--Mirsky theorem
  \citep{eckartApproximationOneMatrix1936} therefore gives
 \begin{align*}
  \inf_{\rank(h)\leq d}\norm{\tilt_0-h}_{L^2(\mu_\prodMeas)}^2
   & =\inf_{\rank(T)\leq d}\norm{K_{\tilt_0}-T}_{\mathrm{HS}}^2  \\
   & =\inf_{\rank(B)\leq d}\norm{\nfr{1}{2}I_4-B}_{\mathrm{F}}^2
  =\frac{4-d}{4}>0,
 \end{align*}
 for $d \in \{1,2,3\}$.
 For a normalized CLIP model with $\dLat=d$, the score factorizes as
 \begin{align*}
  \tilt^{\twoclip}(x_1,x_2)
  =\frac{1}{\tau}\sum_{k=1}^d
  \gNorm_{1,k}(x_1)\gNorm_{2,k}(x_2),
 \end{align*}
 and hence has rank at most $d$. This rank bound is unaffected by the value
 of $\tau$: multiplication by $1/\tau$ changes the singular values but not
 their number.
 Next, we translate this rank obstruction into a measure-level bound via a
 localized version of the same finite-dimensional argument. Define
 \begin{align*}
  z_k\define-\frac45+\frac{2k}{5},\quad k=0,\ldots,4,
  \qquad
  \epsilon\define\frac18,
  \qquad
  \rho\define\frac{1}{48\pi},
 \end{align*}
 and let $I_k\define[z_k-\rho,z_k+\rho]$.
 These constants are chosen so that the intervals $I_k$ are pairwise
  disjoint and $2\epsilon+12\pi\rho=\frac12<\frac58$,
  which ensures $4(2\epsilon+12\pi\rho)<5/2$, as required for the
  Frobenius-norm contradiction below. At the grid points,
 \begin{align*}
  \log r(z_i,z_j)-\log r(z_i,z_0)
  =\frac52\indicator[i=j],
  \qquad i,j\in\{1,\ldots,4\}.
 \end{align*}
 The subtraction of the reference column $z_0$ cancels the normalizing
 constant in $\log r$. Now let $d\in\{1,2,3\}$ be the latent dimension,
 and take arbitrary $x_i\in I_i$, $y_j\in I_j$, $i,j\in\{1,\ldots,4\}$,
 and $y_0\in I_0$. The matrix
 \begin{align*}
  A\define\left[
   \log r^{\twoclip}(x_i,y_j)-
   \log r^{\twoclip}(x_i,y_0)
   \right]_{i,j=1}^4
 \end{align*}
 has entries
 \begin{align*}
  A_{ij}=\frac{1}{\tau}\sum_{k=1}^d
  \gNorm_{1,k}(x_i)
  \left(\gNorm_{2,k}(y_j)-\gNorm_{2,k}(y_0)\right).
 \end{align*}
 Thus $A$ has rank at most $d$, independently of $\tau$. On the grid, the
 target matrix is $M_0=(5/2)I_4$. Since all four singular values equal $5/2$,
 the Eckart--Young--Mirsky
  \citep{eckartApproximationOneMatrix1936}
 theorem gives
 \begin{align*}
  \inf_{\rank(B)\leq d}\norm{M_0-B}_{\mathrm{F}}
  =\frac52\sqrt{4-d}\geq\frac52.
 \end{align*}
 On the other hand, if
 $\abs{\log r^{\twoclip}-\log r}<\epsilon$ at all the points involved,
 then inserting the corresponding target log-densities into each column
 difference gives
 \begin{align*}
  \abs{A_{ij}-(M_0)_{ij}}
  \leq{} &
  \abs{\log r^{\twoclip}(x_i,y_j)-\log r(x_i,y_j)}           \\
         & +\abs{\log r^{\twoclip}(x_i,y_0)-\log r(x_i,y_0)} \\
         & +\abs{\log r(x_i,y_j)-\log r(z_i,z_j)}            \\
         & +\abs{\log r(x_i,y_0)-\log r(z_i,z_0)}.
 \end{align*}
 The first two terms are smaller than $\epsilon$. For the last two terms,
 the constant $-\log Z$ cancels and
 \begin{align*}
  \partial_{x_1}\tilt_0(x_1,x_2)
   & =-\pi\sin\bigl(\pi(x_1-x_2)\bigr)
  -2\pi\sin\bigl(2\pi(x_1-x_2)\bigr),  \\
  \partial_{x_2}\tilt_0(x_1,x_2)
   & =-\partial_{x_1}\tilt_0(x_1,x_2).
 \end{align*}
 Thus both partial derivatives have absolute value at most $3\pi$.
 Hence $\log r$ is $3\pi$-Lipschitz with respect to the $\ell^1$ distance,
 and, since all arguments lie within $\rho$ of their grid points,
 \begin{align*}
  \abs{\log r(x_i,y_j)-\log r(z_i,z_j)}
   & \leq3\pi\bigl(\abs{x_i-z_i}+\abs{y_j-z_j}\bigr)
  \leq6\pi\rho,                                      \\
  \abs{\log r(x_i,y_0)-\log r(z_i,z_0)}
   & \leq3\pi\bigl(\abs{x_i-z_i}+\abs{y_0-z_0}\bigr)
  \leq6\pi\rho.
 \end{align*}
 Combining the four estimates yields, entrywise,
 \begin{align*}
  \abs{A_{ij}-(M_0)_{ij}}
  <2\epsilon+12\pi\rho
  =\frac12.
 \end{align*}
 Consequently, $\norm{A-M_0}_{\mathrm{F}}<4/2 =2$, contradicting the
 preceding Eckart--Young--Mirsky bound. Thus, for every choice of the points, at
 least one of the log-density errors involved is at least $\epsilon$. To
 turn this pointwise alternative into a measure bound, let
 \begin{align*}
  D\define\left\{(x_1,x_2) \in [-1,1]^2 \; \vert \;
  \abs{\log r^{\twoclip}(x_1,x_2)-\log r(x_1,x_2)}
  \geq\epsilon\right\}.
 \end{align*}
 Equivalently, for every such choice of the points, at least one
 has to belong to $D$, thus
 \begin{align*}
  1
  \leq\sum_{i,j=1}^4\indicator_D(x_i,y_j)
  +\sum_{i=1}^4\indicator_D(x_i,y_0).
 \end{align*}
 Now draw all these variables independently, with $x_i$ distributed as
 $\mu_1(\innerEmpty{}\mid I_i)$, $y_j$ as
 $\mu_2(\innerEmpty{}\mid I_j)$, and $y_0$ as
 $\mu_2(\innerEmpty{}\mid I_0)$. Each interval has marginal measure
 $\rho$, and therefore, for example,
 \begin{align*}
  \Expu{}{\indicator_D(x_i,y_j)}
   & =\int_{I_i}\int_{I_j}\indicator_D(x,y)
  \frac{\mu_2(\Dd y)}{\mu_2(I_j)}
  \frac{\mu_1(\Dd x)}{\mu_1(I_i)}
  =\frac{1}{\rho^2}
  \int_{I_i}\int_{I_j}\indicator_D(x,y)\,
  \mu_2(\Dd y)\mu_1(\Dd x)
  =\frac{\mu_\prodMeas(D\cap(I_i\times I_j))}{\rho^2}.
 \end{align*}
 Taking expectations in the pointwise indicator inequality gives
 \begin{align*}
  1 \leq\frac{1}{\rho^2}
  \left(
  \sum_{i,j=1}^4\mu_\prodMeas(D\cap(I_i\times I_j))
  +\sum_{i=1}^4\mu_\prodMeas(D\cap(I_i\times I_0))
  \right)
  \leq\frac{\mu_\prodMeas(D)}{\rho^2},
 \end{align*}
 where the last inequality uses the disjointness of the rectangles. Hence
 $\mu_\prodMeas(D)\geq\rho^2$.

 On $D$, the density ratio lies outside $(\exp(-\epsilon),\exp(\epsilon))$,
 and hence $\abs{r^{\twoclip}/r-1}\geq1-\exp(-\epsilon)$. Since
 $r\geq\exp(-4)$, it follows that
 \begin{align*}
  \norm{r-r^{\twoclip}}_{L^1(\mu_\prodMeas)}
   & \geq\exp(-4)(1-\exp(-\epsilon))\rho^2.
 \end{align*}
 Pinsker's inequality now gives the following lower bound, uniformly over all
 temperatures, parameter dimensions, and encoder parameterizations:
 \begin{align*}
  \KL{\mu}{\mu^{\twoclip}(\innerEmpty{};\thetaVec)}
  \geq\kappa
  \define\frac{\exp(-8)}{2}
  \left(1-\exp\left(-\frac18\right)\right)^2
  \left(\frac{1}{48\pi}\right)^4>0,
  \quad \dLat=d\in\{1,2,3\}.
 \end{align*}

 The latent-dimension threshold is sharp for the normalized CLIP family.
 Define identical normalized encoders $\gNorm_i:[-1,1]\to\Simplex^3$,
 $i=1,2$, by
 \begin{align*}
  \gNorm_i(x_i)\define\frac{1}{\sqrt{2}}
  \bigl(\cos(\pi x_i),\sin(\pi x_i),
  \cos(2\pi x_i),\sin(2\pi x_i)\bigr)^\top.
 \end{align*}
 Then $\norm{\gNorm_i(x_i)}_2=1$ and, with $\tau=1/2$,
 \begin{align*}
  \frac{\IP{\gNorm_1(x_1)}{\gNorm_2(x_2)}}{\tau}
  =\tilt_0(x_1,x_2).
 \end{align*}
 Thus a four-dimensional Fourier dictionary represents $\mu$ exactly,
 even though each modality is one-dimensional.
\end{proof}

\subsection{Contrastive Learning $m>2$}
This subsection proves
 \Cref{thm:noUniversalApproximationClip}, showing that the standard pairwise
 extension of CLIP cannot universally approximate the joint. 
 We then develop a consistent exponential-family representative that preserves all pairwise
 conditionals for use in the next subsection.

\subsubsection{Proof of \Cref{thm:noUniversalApproximationClip}}
\label{thm:noUniversalApproximationClip-PROOF}

\begin{proof}
 Consider $\XX=[-1,1]^m$ equipped with the reference measure
 $\mu_\prodMeas=\Uniform([-1,1])^{\otimes m}$ and define the measure $\mu$
 through
 \begin{align*}
  \frac{\Dd \mu}{\Dd \mu_\prodMeas}(x) = r(x)
  \define \frac{1}{2}(x_1 x_2 x_3 + 2).
 \end{align*}
 Since $r \in \left[\frac{1}{2},\frac{3}{2}\right]$ on $\XX$ and
 $\Expu{\mu_\prodMeas}{x_1 x_2 x_3}=0$, $r$ is a valid probability density.
 Moreover, the tilting function $\tilt=\log r$ is continuous and bounded.

 \emph{Step 1: Pairwise marginals coincide with the product.}
 Integrating $r$ with respect to any one of the first three variables
 eliminates the three-way interaction. It follows that all univariate and
 pairwise marginals of $\mu$ coincide with those of $\mu_\prodMeas$, verifying
 that the reference measure is indeed the product of the marginals of $\mu$.
 Consequently, $\mu_{j3}=\mu_j\otimes\mu_3$ for every $j\neq3$, and
 $\mu_{-3}=\mu_{\prodMeas,-3}
  = \bigotimes_{j\neq3}\mu_j$.

 \emph{Step 2: Conditional density of the pairwise model.}
 Fix arbitrary $\dLat$, $p\in\NN$, a parameter space
  $\Theta\subseteq\RR^p$ and associated normalized encoder parameterization,
  $\thetaVec\in\Theta$, and $\tau>0$.
 For $j\neq3$, define the pairwise tilting terms involving modality $3$ by
 \begin{align*}
  f_{j3}(x_j,x_3;\thetaVec)
  \define
  \IP{\gNorm_j(x_j;\theta_j)}{\gNorm_3(x_3;\theta_3)}/\tau,
 \end{align*}
 and let
 \begin{align*}
  G(x_{-3},x_3;\thetaVec)
  \define \sum_{j\neq3} f_{j3}(x_j,x_3;\thetaVec).
 \end{align*}
 All terms in \eqref{eq:ClipTiltingFunctionGeneral} that do not involve
 $x_3$ cancel upon conditioning on $x_{-3}$. Hence
 \begin{align*}
  \frac{\Dd\mu^{\mclip}_{x_3\vert x_{-3}}(\innerEmpty{};\thetaVec)}{\Dd\mu_3}
  (x_3\vert x_{-3})
  = \frac{\exp(G(x_{-3},x_3;\thetaVec))}{Z(x_{-3};\thetaVec)},
 \end{align*}
 where
 \begin{align*}
  Z(x_{-3};\thetaVec)
  \define \int_{-1}^1 \exp(G(x_{-3},t;\thetaVec))\;\mu_3(\Dd t).
 \end{align*}

 \emph{Step 3: Reduction to a KL against the uniform measure.}
 By the pairwise marginal identities from Step~1, we obtain
 \begin{align*}
  \Expu{\xx\sim\mu}{G(x_{-3},x_3;\thetaVec)}
  = \sum_{j\neq3}\Expu{(x_j,x_3)\sim\mu_j\otimes\mu_3}{
   f_{j3}(x_j,x_3;\thetaVec)}
  = \Expu{\xx\sim\mu_\prodMeas}{G(x_{-3},x_3;\thetaVec)},
 \end{align*}
 and expanding the conditional KL divergences using this identity and
 $\mu_{-3}=\mu_{\prodMeas,-3}$ yields
 \begin{align*}
   & \Expu{x_{-3}\sim\mu_{-3}}{
  \KL{\mu_{x_3\vert x_{-3}}}{
  \mu^{\mclip}_{x_3\vert x_{-3}}(\innerEmpty{};\thetaVec)}} \\
   & \quad =
  \Expu{x_{-3}\sim\mu_{-3}}{
   \KL{\mu_{x_3\vert x_{-3}}}{\mu_3}}
  + \Expu{x_{-3}\sim\mu_{-3}}{
  \KL{\mu_3}{
  \mu^{\mclip}_{x_3\vert x_{-3}}(\innerEmpty{};\thetaVec)}} \\
   & \quad \geq
  \Expu{x_{-3}\sim\mu_{-3}}{
   \KL{\mu_{x_3\vert x_{-3}}}{\mu_3}}.
 \end{align*}
 The key step of this proof, is that now the
  right-hand side no longer depends on
  $\thetaVec$ or any other model choice, so it remains only to lower-bound a
  quantity determined by $\mu$.

 \emph{Step 4: Pinsker's inequality.}
 For fixed $x_{-3}$, set $a(x_{-3})\define x_1x_2$. The true conditional and
 the marginal $\mu_3$ have Lebesgue densities $(a(x_{-3})x_3+2)/4$ and $1/2$,
 respectively. A direct calculation therefore yields
 \begin{align*}
  \DTV{\mu_{x_3\vert x_{-3}}}{\mu_3}
   = \frac{1}{2}\int_{-1}^{1}
   \abs*{\frac{a(x_{-3})x_3+2}{4}-\frac12} \dx_3
   = \frac{\abs*{a(x_{-3})}}{8},
 \end{align*}
 so Pinsker's inequality gives
 \begin{align*}
  \KL{\mu_{x_3\vert x_{-3}}}{\mu_3}
  \geq 2\DTV{\mu_{x_3\vert x_{-3}}}{\mu_3}^2
  = \frac{a(x_{-3})^2}{32}.
 \end{align*}
 Here $a(x_{-3})=x_1x_2$ is determined by the conditioning variables rather
 than chosen. Since $x_1$ and $x_2$ are independent and uniformly distributed
 on $[-1,1]$, taking expectations and using Step~3 gives
 \begin{align*}
  \Expu{x_{-3}\sim\mu_{-3}}{
  \KL{\mu_{x_3\vert x_{-3}}}{
  \mu^{\mclip}_{x_3\vert x_{-3}}(\innerEmpty{};\thetaVec)}}
  \geq \frac{1}{32}\Exp{x_1^2}\cdot \Exp{x_2^2}
  = \frac{1}{32}\cdot\frac{1}{9}
  = \frac{1}{288}.
 \end{align*}
 Since all model choices above were arbitrary, this bound is uniform over
 latent dimensions, parameter dimensions, encoder parameterizations,
 parameters, and temperatures. Moreover,
 \Cref{thm:JointVsFullConditionals}
 (applied to $i=3$) yields
 \begin{align*}
  \KL{\mu}{\mu^{\mclip}(\innerEmpty{};\thetaVec)} \geq \frac{1}{288}>0.
 \end{align*}
 Thus, the choice $\kappa\define 1/288$ proves the claim.
\end{proof}

\subsubsection{Consistent $L^1$ Representatives for Pairwise Conditionals}

\begin{lemma}
 \label[lemma]{lem:L1PairwiseConditionalsExponentialFamily}
 Let $\XX \subset \RR^d$, not necessarily compact.
 Assume the data generating measure $\mu$ has a tilting $r:\XX \to [0,\infty]$
 with respect to $\mu_\prodMeas$ satisfying $\tilt = \log r \in L^1(\XX, \mu)$.

 Then there exists a measure $\mu^\star \in \Prob(\XX)$ with a tilting of
 the form
 \begin{align}
  \frac{\Dd \mu^\star}{\Dd \mu_\prodMeas}(\xx)
  = \exp\left( \sum_{1\leq i < j \leq m} \tilt^\star_{ij}(x_i, x_j) - \log(Z^\star)\right)
  \tag{\ensuremath{\mu^\star}-Tilting}
  \label{eq:MuStarDensityConsistentRepresentationPairwise}
 \end{align}
 for measurable functions $\tilt_{ij}^\star : \XX_i \times \XX_j \to \RR$,
 with $\tilt_{ij} \in L^1(\XX, \mu_{ij})$, for $(i,j)\in \setN{m}^2$, $i\neq j$
 and a normalization constant $Z^\star > 0$, such that
 for all pairwise marginals
 \begin{align*}
  \mu^\star_{ij} = \mu_{ij}
 \end{align*}
 holds. Consequently, by the chain rule of the KL divergence, the pairwise conditionals
 \begin{align*}
  \Expu{x_j \sim \mu_j}{\KL{\mu_{x_i\vert x_j}}{\mu^\star_{x_i\vert x_j}}} = 0
 \end{align*}
 are perfectly matched.
\end{lemma}
\begin{proof}
 The main technically challenge is to find a \emph{consistent} representative, i.e.
 formally we are interested in finding a solution to the following (constrained) projection problem
 \begin{align}
  \min_{\nu \in \CC_\mu} \KL{\nu}{\mu_\prodMeas}, \text{ s.t. }
  \CC_\mu \define \{\nu \in \PP(\XX) \; \vert \; \nu_{ij} = \mu_{ij} \text{ for all } i < j\},
  \tag{Marg-Proj}
  \label{eq:ConstrainedKlDivergenceSetPairwiseMarginals}
 \end{align}
 where the constraint set contains all probability measures matching the pairwise marginals
 of the data generating
 measure $\mu$. This set $\CC_\mu$ is non-empty, since $\mu \in \CC_\mu$,
 and, since we assume $\log r \in L^1(\XX, \mu)$, it contains with the measure $\mu$ itself
 an element in $\CC_\mu \cap \{\KL{\innerEmpty{}}{\mu_\prodMeas} < \infty\}$.
 Additionally,
 the constraint is linear, making $\CC_\mu$ convex.
 Moreover, $\CC_\mu$ is variation-closed with respect to the total variation (TV). I.e. for
 a sequence $(\nu_n)_{n\in \NN} \subset \CC_\mu$ with $\nu_n \to \nu$ in TV,
 it holds $\nu \in \CC_\mu$, since marginalization is a contraction in TV.

 Theorem 2.1 in \cite{csiszar$I$DivergenceGeometryProbability1975} implies
 that there exists a (unique) solution $\mu^\star$ for
 \eqref{eq:ConstrainedKlDivergenceSetPairwiseMarginals}. By the extension of
 Corollary 3.1 in \cite{csiszar$I$DivergenceGeometryProbability1975} to
 prescribed marginals on product spaces (cf. \cite[p.
  151]{csiszar$I$DivergenceGeometryProbability1975} and see Cor. 3.2 for the
 bi-modal case), $\mu^\star$ has a density with respect to $\mu_\prodMeas$
 of the form
 \begin{align*} \frac{\Dd \mu^\star}{\Dd \mu_\prodMeas}(\xx) =
  \exp\left(\sum_{1\leq i < j \leq m} \tilt^\star_{ij}(x_i, x_j) -
  \log(Z^\star)\right),\end{align*}
 where $\tilt^\star_{ij} : \XX_i \times \XX_j \to \RR$ are
 $\mu_{ij}$-measurable, $\tilt_{ij}^\star \in L^1(\XX_i \times \XX_j,
  \mu_{ij})$, and a constant $Z^\star > 0$; i.e.
 \eqref{eq:MuStarDensityConsistentRepresentationPairwise} holds.
 This composition holds up to a
 set $N \subset \XX$ which satisfies $\nu(N) = 0$ for every $\nu \in \CC_\mu
  \cap \{\KL{\innerEmpty{}}{\mu} < \infty\}$, in particular we have $\mu(N) =
  0$, which is important for the next step when we apply the KL divergence
 chain rule.\footnote{ Note that technically Corollary 3.1 in
  \cite{csiszar$I$DivergenceGeometryProbability1975} yields a decomposition
  of the form $\frac{\Dd \mu^\star}{\Dd \mu_\prodMeas} = \prod_{1\leq i < j
    \leq m} r^\star_{ij}(x_i, x_j)$, where each term $(i,j) \in \setN{m}^2$,
  $i < j$, satisfies $\log r^\star_{ij}
   \in L^1(\XX_i \times \XX_j, \mu_{ij})$, via log-exp
  transformation $\tilt^\star_{ij} \define \log r^\star_{ij}$. }

 Since the measure $\mu^\star$ has full support by construction except for a
 set $N$ which fulfills $\mu(N) = 0$, we have $\mu \ll \mu^\star$. Applying
 the chain rule for the KL divergence \cite[Thm.
  2.5.3]{coverELEMENTSINFORMATIONTHEORY} yields
 \begin{align*}
  \KL{\mu_{ij}}{\mu^\star_{ij}} = \KL{\mu_{j}}{\mu^\star_{j}}
  + \Expu{x_j \sim \mu_j}{\KL{\mu_{x_i\vert x_j}}{\mu^\star_{x_i\vert x_j}}}
 \end{align*}
 and rearranging leads to
 \begin{align*}
  \Expu{x_j \sim \mu_j}{\KL{\mu_{x_i\vert x_j}}{\mu^\star_{x_i\vert x_j}}}
  =
  \KL{\mu_{ij}}{\mu^\star_{ij}} - \KL{\mu_{j}}{\mu^\star_{j}}
  \leq  \KL{\mu_{ij}}{\mu^\star_{ij}}  = 0.
 \end{align*}
 In total, the measure $\mu^\star$ is an instance of the exponential family, which
 we can approximate with our CLIP architecture, and
 perfectly interpolates all pairwise conditionals of the original measure $\mu$.
\end{proof}

\subsection{Proof of \Cref{thm:UniversalApproxClipPairwiseConditionals}}
\label{thm:UniversalApproxClipPairwiseConditionals-PROOF}

This subsection combines the consistent pairwise representative, the
 bimodal CLIP approximation result, and a block-orthogonal normalized-encoder
 construction to prove simultaneous universal approximation of all pairwise
 conditionals.
 
\begin{proof}
 Our proof consists of three steps:  First, we use
 \Cref{lem:L1PairwiseConditionalsExponentialFamily} to get a proxy target
 $\mu^\star$ for our
 approximation guarantee. Second, $\mu^\star$ has density with a separable structure
 for which we can apply our results of the bi-modal CLIP individually and third,
 we translate these pairwise approximation guarantees using
 \Cref{prop:FuncToMeasCompactSupport} and
 \Cref{thm:JointVsPairwiseConditionals} into a guarantee for the pairwise
 conditionals, since $\mu^\star$ perfectly matches
 the pairwise conditionals of our target measure $\mu$. Define the accuracy
 $\epsilon^\prime \define \epsilon / (2 \cdot \binom{m}{2})$.

 \high{(1) Consistent Representative as Proxy.} Using
 \eqref{eq:MuStarDensityConsistentRepresentationPairwise} of
 \Cref{lem:L1PairwiseConditionalsExponentialFamily} yields a tilting function
 of the form $\sum_{i<j} \tilt^\star_{ij}(x_i, x_j)$, for $\tilt_{ij}^\star$
 measurable and $\tilt_{ij}^\star \in L^1(\XX_i \times \XX_j;\mu_{ij})$.
 For each pair $(i,j) \in \setN{m}^2$, $i < j$, invoking
 \Cref{lem:LusinsTheoremGeneric} on the marginal of $\mu^\star_{ij}$
 yields continuous representatives
 $\tiltContArg{ij} : \XX_i \times \XX_j \to \RR$ such that
 \begin{align*}
  \norm{\tilt^\star_{ij} - \tiltContArg{ij}}_{L^1(\XX_i \times \XX_j; \mu_{ij})}
  < \epsilon^\prime
  \quad \text{and} \quad
  \int_{\XX_i \times \XX_j} \exp(\tiltContArg{ij}) \dd \mu_{\prodMeas, ij}
  \leq 1 + \epsilon^\prime.
 \end{align*}

 \high{(2) Block-orthogonal Encoder Construction.}
 For every pair $(i,j) \in \setN{m}^2$, $i < j$,
 by \Cref{lem:DensityClipUnnormalizedEncoders}
 applied to the bi-modal problem
 of approximating $\tiltContArg{ij}$, we can find a latent dimension
 $d_{ij} \in \NN$ and continuous (unnormalized) encoders
 $g_i^{(ij)}, g_{j}^{(ij)}: \XX_i, \XX_j \to
  \Simplex^{d_{ij} -1}$ with
 \begin{align*}
  \norm{\tiltContArg{ij} - \IP{g_{i}^{(ij)}}{g_{j}^{(ij)}}
  }_{\infty, \XX_i \times \XX_j} < \epsilon^\prime,
 \end{align*}
 which we will use to construct a consistent, block-orthogonal
 encoder structure. We have $\binom{m}{2}$ pairs of encoders,
 which we order as $(1,2), (1,3), \ldots, (1,m), (2,3), \ldots, (m-1,m)$.
 Define $\dLat \define m + \sum_{1\leq i < j \leq m} d_{ij}$
 and we interpret the latent space
 \begin{align*}
  \RR^{\dLat} \cong \RR^{d_{12}} \oplus \RR^{d_{13}} \oplus \cdots \oplus
  \RR^{d_{m-1,m}} \oplus \RR^m
 \end{align*}
 as a concatenation of blocks $B_{ij}$, $i < j$, with an additional padding
 block of size $m$. To ensure normalized encoders, we construct, similar to
 the proof of \Cref{thm:NormalizedTwoClipUniversalApproximation}, a rescaling
 and padding of the form
 \begin{align*}
  G       \define \max_{i < j} \max\{\norm{g_i^{(ij)}}_{\infty, \XX_i },
  \norm{g_j^{(ij)}}_{\infty, \XX_j}\},
  \quad d       \define \max_{i < j} d_{ij},
  \quad \alpha  \define \frac{1}{2 \sqrt{(m-1) \cdot d \cdot G^2 + 1}},
 \end{align*}
 and we define the normalization
 \begin{align*}
  p_i(x_i) \define \sqrt{1 - \alpha^2 \sum_{j\neq i} \norm{g_i^{(ij)}(x_i)}_{2}^2},
 \end{align*}
 where the choice of $\alpha$ ensures that $p_i \geq \sqrt{3}/2 > 0$.
 We formally define the (normalized) block-encoders as
 \begin{align*}
  \gNorm_{i}(x_i;\thetaHat_i) \define
  \alpha \cdot
  \left(\left[\delta_{i\in\{k,l\}}
  \cdot g_i^{(kl)}(x_i;\thetaHat^{(kl)}_i)\right]_{1 \leq k < l \leq m} \right)
  \oplus p_i(x_i) \cdot e_i \in \RR^{\dLat},
 \end{align*}
 where $\delta_{i \in \{k,l\}} \define \indicator\left[i \in \{k,l\}\right]$,
 $e_i \in \RR^m$ is
 the standard $i$-th basis vector and
 we collect all pair-specific parameters involving the modality $i$
 into $\thetaHat_i \define
  \big(\thetaHat_i^{(ij)}\big)_{j>i}
  \oplus \big(\thetaHat_i^{(ji)}\big)_{j<i}$.
 Intuitively and by omitting the parameters for brevity,
 one can illustrate this construction in matrix form as
 \begin{align*}
  \begin{pmatrix}
   g_1    \\
   g_2    \\
   g_3    \\
   \vdots \\
   g_m
  \end{pmatrix}
  =
  \begin{pmatrix}
   g_1^{(12)} & g_1^{(13)} & \ldots &        & g_1^{(1m)} & 0           & 0          & \ldots & 0             & \ldots \\
   g_2^{(12)} & 0          & \ldots &        & 0          & g_2^{(23)}  & g_2^{(24)} & \ldots & g_{2}^{(2,m)} & \ldots \\
   0          & g_3^{(13)} & 0      & \ldots & 0          & g_3^{(2,3)} & 0          & \ldots & 0             & \ldots \\
              &            & \ddots &        &            &             &            &        &               &        \\
   0          &            & \ldots & 0      & g_m^{(1m)} & 0           & \ldots     & 0      & g_m^{(2,m)}   & \ldots
  \end{pmatrix}
 \end{align*}
 which ensures that the inner product of the block encoders exactly mimics the interactions of the pairwise
 encoders we constructed.
 We define the shared temperature $\tau \define \alpha^2 > 0$, and the
 construction of $\gNorm_i$ ensures $\norm{\gNorm_{i}(x_i;\thetaHat_i)}_{2} = 1$ and by
 the block-wise structure we obtain
 \begin{align*}
   & \frac{1}{\tau}\IP{\gNorm_{i}(x_i;\thetaHat_i)}{\gNorm_{j}(x_j;\thetaHat_j)}    \\
   & = \frac{\alpha^2}{\tau} \sum_{1\leq k < l \leq m}
  \indicator\left[i \in \{k,l\}\right] \cdot \indicator\left[j \in \{k,l\}\right]
  \cdot \IP{g_i^{(kl)}(x_i;\thetaHat^{(kl)}_i)}{\gNorm_{j}^{(kl)}(x_j;\thetaHat^{(kl)}_j)}
  + \frac{1}{\tau} p_i(x_i) p_j(x_j) \delta_{ij}                                    \\
   & = \IP{g_i^{(ij)}(x_i;\thetaHat^{(ij)}_i)}{g_j^{(ij)}(x_j;\thetaHat^{(ij)}_j)},
 \end{align*}
 since the product of the two indicator functions is $\indicator\left[\{i,j\}
   \subseteq \{k,l\}\right]$, which for $i\neq j$ forces $\{k,l\} = \{i,j\}$,
 we recover exactly the unnormalized encoder inner product.
 By stacking all parameters together as $\thetaVecHat \define
  (\thetaHat^\top_1, \ldots \thetaHat^\top_m)^\top$,
 the total approximation guarantee of the tilting functions $\tiltContArg{ij}$
 is given by
 \begin{align*}
  \norm*{\sum_{1 \leq i < j \leq m} \tiltContArg{ij}
   - \sum_{1 \leq i < j \leq m} \frac{\IP{\gNorm_i(\innerEmpty{};\thetaHat_i)}{
     \gNorm_j(\innerEmpty{};\thetaHat_j)}}{\tau}
  }_{\infty, \XX}
   & \leq
  \sum_{1 \leq i < j \leq m} \norm*{\tiltContArg{ij}
   - \frac{\IP{\gNorm_i(\innerEmpty{};\thetaHat_i)}{\gNorm_j(\innerEmpty{};\thetaHat_j)}}{\tau}
  }_{\infty, \XX_i \times \XX_j}          \\
   & < \binom{m}{2} \cdot \epsilon^\prime
 \end{align*}
 and \Cref{prop:FuncToMeasCompactSupport} yields that
 \begin{align*}
  \KL{\mu^\star}{\mu^{\mclip}(\innerEmpty{};\thetaVecHat)}
  \leq \binom{m}{2} (\epsilon^\prime + \epsilon^\prime)
  = \epsilon,
 \end{align*}
 i.e. $\mu^\mclip(\innerEmpty{};\thetaVecHat)$ universally approximates (the joint distribution of)
 the proxy target $\mu^\star$.

 \high{(3) Translation to Pairwise Conditionals Approximation.}
 By \Cref{lem:L1PairwiseConditionalsExponentialFamily}, it holds
 $\mu^\star_{ij} = \mu_{ij}$ as measures on $\XX_i \times \XX_j$
 and by integrating over $x_i$, we obtain $\mu_j^\star = \mu_j$ as
 a measure on $\XX_j$. By the disintegration theorem,
 we obtain
 \begin{align*}
  \mu^\star_{x_i\vert x_j}(\innerEmpty{}\vert x_j) =
  \mu_{x_i\vert x_j}(\innerEmpty{}\vert x_j)
 \end{align*}
 for $\mu_j$-almost all $x_j$, as probability measures on $\XX_i$.
 Since $\mu^{\mclip}(\innerEmpty;\thetaVecHat)$ has a
 strictly positive density, $\mu^\star \ll \mu^{\mclip}(\innerEmpty{};\thetaVecHat)$ holds.
 Invoking \Cref{thm:JointVsPairwiseConditionals} yields
 \begin{align*}
  \Expu{x_j \sim \mu_j}{\KL{\mu_{x_i\vert x_j}}{
    \mu^{\mclip}_{x_i\vert x_j}(\innerEmpty{};\thetaVecHat)}}
   & = \Expu{x_j \sim \mu^\star_j}{\KL{\mu^\star_{x_i\vert x_j}}{
  \mu^{\mclip}_{x_i\vert x_j}(\innerEmpty{};\thetaVecHat)}}                      \\
   & \leq \KL{\mu^\star}{\mu^{\mclip}(\innerEmpty{};\thetaVecHat)} \leq \epsilon
 \end{align*}
 for all pairs $i < j$, concluding the proof.
\end{proof}

\subsection{Proof of \Cref{thm:UniversalApproximationHadamardClip}}
\label{thm:UniversalApproximationHadamardClip-PROOF}

This, final, subsection proves universal approximation for
Hadamard-CLIP by establishing density of fully separable products through
Stone--Weierstrass, approximating their factors with the encoder families, and
then enforcing encoder normalization through rescaling and orthogonal
padding.
\begin{proof} We structure this proof into three steps: first, we show a density-type result
 of a more generic sub-algebra using the Stone-Weierstraß theorem. Second, we prove
 that we can approximate this sub-algebra using our parametric family and third,
 we extend this approximation to normalized encoders.

 \high{(1) Sub-algebra Argument.} To use the Stone Weierstraß argument in
 \Cref{thm:StoneWeierstrass}, we define the set
 \begin{align*}
  A^{\hadamard} \define \left\{
  \sum_{k=1}^K \omega_k \prod_{i=1}^m  \phi_{i,k}(x_i) \; \big\vert \;
  K\in \NN, \phi_{i,k} \in C^0(\XX_i; \RR), i=1,\ldots, m; k=1,\ldots, K;
  \omega \in \RR^{K}
  \right\}
 \end{align*}
 and show that $A^{\hadamard}$ is a subalgebra, contains constants and separates points.
 Closedness under multiplication with scalar and addition is clear. For two elements
 $h, h^\prime \in A_{\hadamard}$ of the form
 \begin{align*}
  h = \sum_{k=1}^K \omega_k \prod_{i=1}^{m} \phi_{i,k}, \quad
  h^\prime = \sum_{l=1}^{L} \omega_l^\prime \prod_{i=1}^{m} \psi_{i,l}
 \end{align*}
 their product is
 \begin{align*}
  h \cdot h^\prime  = \sum_{k=1}^{K} \sum_{l=1}^{L} \omega_k \omega^\prime_l
  \prod_{i=1}^{m} \phi_{i,k} \cdot \prod_{i=1}^{m}\psi_{i,l}
  = \sum_{k=1}^{K} \sum_{l=1}^{L} \underbrace{(\omega_k
   \cdot \omega_l^\prime)}_{\defineRev \tilde{\omega}_{kl}}
  \prod_{i=1}^{m} \underbrace{\phi_{i,k} \cdot \psi_{i,l}}_{\defineRev \chi_{i, (k,l)}
   \in C^0(\XX_i)}
 \end{align*}
 which is a sum of $K\cdot L$ fully separable monomials. Since products of continuous
 functions are continuous, we have that each $\chi_{i,(k,l)} \in C^0(\XX_i;\RR)$, so
 in total $h \cdot h^\prime \in A^{\hadamard}$. Since $A^{\hadamard}$ trivially
 contains a non-zero constant function and separates points, with
 \Cref{thm:StoneWeierstrass} the set $A^{\hadamard}$ is dense in $C^0(\XX;\RR)$.

 \high{(2) Approximation Guarantee.} For an arbitrary $\epsilon > 0$ and a continuous
 function $\tilt \in C^0(\XX;\RR)$, we can pick $K\in \NN$, a weight vector $\omega \in \RR^K$
 and continuous functions $\phi_{i,k}$, $i=1,\ldots, m$, $k=1,\ldots, K$ such that
 \begin{align*}
  \norm*{\tilt - \sum_{k=1}^{K} \omega_k \prod_{i=1}^{m} \phi_{i,k}}_{\infty, \XX} < \frac{\epsilon}{2}.
 \end{align*}
 With constants $F \define \max_{i,k} \norm{\phi_{i,k}}_{\infty, \XX}$,
 $W \define \norm{\omega}_{\infty, \XX}$,
 $\delta \define \min \{1, \frac{\epsilon}{2 K \cdot W \cdot m (F+1)^{m-1}}\}$
 and by using
 \UAAssref{ass:DensityParametricFamilyInContinuousFunctions}, we can find parameters
 $\thetaHat_{i,k}$ and corresponding (unnormalized) encoders with
 \begin{align*}
  \norm{\phi_{i,k} - g_{i,k}(\innerEmpty{};\thetaHat_{i,k})}_{\infty, \XX} < \delta.
 \end{align*}
 For a single monomial $k \in \setN{K}$, we can use a telescoping product trick to obtain
 \begin{align*}
   & \abs*{\prod_{i=1}^m \phi_{i,k}(x_i) - \prod_{i=1}^{m} g_{i,k}(x_i; \thetaHat_{i,k})} \\
   & =
  \abs*{\sum_{j=1}^m \left(\prod_{i < j} g_{i,k}(x_i;\thetaHat_{i,k})\right)
   \left(\phi_{j,k}(x_i) - g_{i,k}(x_i;\thetaHat_{i,k})\right)
   \left(\prod_{i > j} \phi_{i,k}(x_i)\right)
  }                                                                                       \\
   & \leq \sum_{j=1}^{m} (F+\delta)^{j-1} \cdot \delta \cdot F^{m-j}
  \leq m (F+\delta)^{m-1}\cdot \delta
 \end{align*}
 for every $x_i \in \XX_i$. Summing over all $K$ monomials with weights $\omega$ yields
 \begin{align*}
  \abs*{\sum_{k=1}^{K}\omega_k\prod_{i=1}^m \phi_{i,k}(x_i) -
   \sum_{k=1}^{K} \omega_k \prod_{i=1}^{m} g_{i,k}(x_i; \thetaHat_{i,k})}
  \leq K \cdot W \cdot m (F+\delta)^{m-1} \cdot \delta < \frac{\epsilon}{2},
 \end{align*}
 which shows $\epsilon$-closeness in the supremum norm via the triangle inequality.

 \high{(3) Normalized Encoders.} We can apply the same rescaling trick combined with
 orthogonal padding, i.e. define
 \begin{align*}
  \alpha \define \frac{1}{2 \sqrt{K} (F + \delta) + 1} > 0
 \end{align*}
 which guarantees that
 $\alpha^2 \sum_{k=1}^{K} \phi_{i,k}^2(x_i) \leq \alpha K (F+\delta)^2 < \frac{1}{4}$
 for all $x_i \in \XX_i$ and all $i=1,\ldots, m$. In this case, we need an orthogonal
 padding of $m$, i.e. with $\dLat \define K + m$ and define
 \begin{align*}
  \gNorm_i(x_i;\thetaHat_i) \define
  \left(\alpha g_{i,1}(x_i; \thetaHat_{i,1}), \ldots, \alpha
  g_{i,K}(x_i;\thetaHat_{i,K}), (\delta_{ji})_{j\in \setN{m}} \cdot
  \sqrt{1- \alpha^2 \sum_{k=1}^K g^2_{i,k}(x_i;\thetaHat_{i,k})}\right)^\top \in \RR^{\dLat},
 \end{align*}
 where $(\delta_{ij})_{j\in \setN{m}}$ is the Kronecker delta vector of dimension
 $m$ with an entry of value one at position $i$ and
 $\thetaHat_i \define (\thetaHat^\top_{i,1}, \ldots, \thetaHat^\top_{i,K})^\top$
 is vector with the concatenated parameters.
 By construction $\norm{\gNorm_{i}(x_i;\thetaHat_i)}_2 = 1$
 and for the hadamard product of the normalized encoders we have
 for $k\leq K$
 \begin{align*}
  \prod_{i=1}^{m} (\gNorm_{i}(x_i,\thetaHat_{i}))_k
  = \alpha^m \prod_{i=1}^{m} g_{i,k}(x_i; \thetaHat_{i,k}),
 \end{align*}
 and for $k > K$ at least one factor of the padding coordinates is zero so
 the entire product vanishes. Under slight abuse of notation, we can absorb
 the weight via
 \begin{align*}
  \omega_k \define \begin{cases}
                    \alpha^{-m} \omega_k, & \text{ for } k\leq K \\
                    0,                    & \text{ for } k> K
                   \end{cases},
 \end{align*}
 which exactly recovers the unnormalized approximation. The final approximation
 result of measures is a consequence of invoking \Cref{result:UniversalApproxJointNetwork}
 or \Cref{result:UniversalApproxJointNetworkTightness}, respectively,
 concluding the proof.
\end{proof}
\end{document}